\documentclass[twoside, 11pt]{article}

\usepackage{jmlr2e}

\usepackage{amsthm}
\usepackage{psfrag}
\usepackage{url}

\usepackage{amsmath}

\usepackage{natbib}

\usepackage{array}
\usepackage{bm}
\usepackage{color}

\usepackage{comment}
\usepackage{graphicx}
\usepackage{graphics,natbib}
\numberwithin{equation}{section}
\theoremstyle{plain}

\newcommand{\cAhat}{\widehat{\cA}}
\newcommand{\cA}{\mathcal{A}}

\newcommand{\cX}{\mathcal{X}}

\newcommand{\kHat}{{k_{\min}}}

\newcommand{\R}{{\rm I}\kern-0.18em{\rm R}}
\newcommand{\p}{{\rm I}\kern-0.18em{\rm P}}
\newcommand{\E}{{\rm I}\kern-0.18em{\rm E}}
\newcommand{\1}{{\rm 1}\kern-0.24em{\rm I}}

\newcommand{\n}{\mathcal{N}}

\newcommand{\argmin}{\mathop{\mathrm{argmin}}}

\newtheorem{TH1}{Theorem}[section]
\newtheorem{prop}{Proposition}[section]

\newtheorem{lem}{Lemma}[section]

\newtheorem{defin}{Definition}[section]

\newtheorem{rem}{Remark}[section]
\newtheorem{assumption}{Assumption}

\newcommand{\Lr}{r}

\newcommand{\uderbar}[1]{\underset{\raise0.3em\hbox{$\smash{\scriptscriptstyle-}$}}{#1}}
\newcommand{\Ca}{\widehat{C}_{\alpha}}

\newcommand{\phiStar}{{\phi^*}}

\newcommand{\Fr}{F_{0,\hat \Lr}}

\newcommand{\chern}{\text{chern}}
\newcommand{\nsn}{\text{NSN}$^2$}
\newcommand{\psn}{\text{PSN}$^2$}
\newcommand{\nn}{\text{NN}$^2$}
\newcommand{\pn}{\text{PN}$^2$}
\newcommand{\nsns}{\text{NSN}$^2$\text{ }}
\newcommand{\psns}{\text{PSN}$^2$\text{ }}
\newcommand{\nns}{\text{NN}$^2$\text{ }}
\newcommand{\pns}{\text{PN}$^2$\text{ }}

\makeatletter
\newcommand{\distas}[1]{\mathbin{\overset{#1}{\kern\z@\sim}}}%
\newsavebox{\mybox}\newsavebox{\mysim}
\newcommand{\distras}[1]{%
  \savebox{\mybox}{\hbox{\kern3pt$\scriptstyle#1$\kern3pt}}%
  \savebox{\mysim}{\hbox{$\sim$}}%
  \mathbin{\overset{#1}{\kern\z@\resizebox{\wd\mybox}{\ht\mysim}{$\sim$}}}%
}
\makeatother
\def\boxit#1{\vbox{\hrule\hbox{\vrule\kern6pt\vbox{\kern6pt#1\kern6pt}\kern6pt\vrule}\hrule}}

\begin{document}
%\textbf{Question}
%\begin{enumerate}
%\item ``\textbf{Proof}[Proof of Lemma A.1]" in the appendix: two "Proof"s a bit redundant??
%\item  
%\end{enumerate}
\title{Neyman-Pearson Classification under High-Dimensional Settings}

\author{\name Anqi Zhao  \email anqizhao@fas.harvard.edu  \\
        \addr  Department of Statistics\\
         Harvard University
\AND 
  \name Yang Feng  \email yangfeng@stat.columbia.edu \\
       \addr Department of Statistics\\
       Columbia University
\AND 
   \name  Lie Wang \email liewang@math.mit.edu \\
       \addr  Department of Mathematics\\
       Massachusetts Institute of Technology 
\AND
\name Xin Tong \email xint@marshall.usc.edu \\
       \addr Department of Data Sciences and Operations \\
       Marshall Business School\\
       University of Southern California
       %Los Angeles, CA 90089, USA
}

\editor{\quad}
%\editor{John Shawe-Taylor}

\maketitle

%\begin{aug}
%\author{\fnms{Anqi} \snm{Zhao}\thanksref{t1}\ead[label=e1]{anqizhao@fas.harvard.edu}},
%\author{\fnms{Yang} \snm{Feng}\thanksref{t2}\ead[label=e2]{yangfeng@stat.columbia.edu }},
%\author{\fnms{Lie} \snm{Wang}\thanksref{t3}\ead[label=e3]{liewang@math.mit.edu}}
%\and
%\author{\fnms{Xin} \snm{
%Tong}\thanksref{t4}
%\ead[label=e4]{xint@marshall.usc.edu}}
%%\thankstext{t1}{Anqi Zhao is Graduate Student, Department of Statistics, Harvard University.}
%\thankstext{t2}{Partially supported by NSF grant DMS-1308566.}
%\thankstext{t4}{Partially supported by the University of Southern California Zumberge Individual Fund 2014-2015, and USC Marshall summer research funding.}
%\runauthor{A. Zhao, Y. Feng, L. Wang, X. Tong}
%
%\affiliation{Harvard University\thanksmark{t1}, Columbia University\thanksmark{t2}, MIT\thanksmark{t3}, and USC\thanksmark{t4}}
%
%\address{Department of Statistics,\\ Harvard University, \\Cambridge, MA 02138.\\\printead{e1}}
%\address{Department of Statistics, \\ Columbia University, \\NY, 10027. \\\printead{e2}}
%\address{Department of Mathematics,\\ Massachusetts Institute of Technology, \\Cambridge, MA, 02139. \\\printead{e3}}
%\address{Department of Data Sciences and Operations, \\Marshall School of Business, \\University of Southern California, CA 90089.\\\printead{e4}}
%\end{aug}

\begin{abstract}
Most existing binary classification methods target on the optimization of the overall classification risk and may fail to serve some real-world applications such as cancer diagnosis, where users are more concerned with the risk of misclassifying one specific class than the other.
Neyman-Pearson (NP) paradigm was introduced in this context as a novel statistical framework for handling asymmetric type I/II error priorities. It seeks classifiers with a minimal type II error and a constrained type I error under a user specified level.
This article is the first attempt to construct classifiers with guaranteed theoretical performance under the NP paradigm in high-dimensional settings. Based on the fundamental Neyman-Pearson Lemma, we used a plug-in approach to construct NP-type classifiers for Naive Bayes models. 
The proposed classifiers satisfy the NP oracle inequalities, which are natural NP paradigm counterparts of the oracle inequalities in classical binary classification. Besides their desirable theoretical properties, we also demonstrated their numerical advantages in prioritized error control via both simulation and real data studies.
\end{abstract}

%\begin{keyword}[class=MSC]
%\kwd[Primary ]{62H30}
%\kwd[; secondary ]{62F07}
%\kwd{62G07}
%\kwd{62G30}
%\end{keyword}
%
%\begin{keyword}
%\kwd{classification} 
%\kwd{high dimension} 
%\kwd{Naive Bayes} 
%\kwd{Neyman-Pearson paradigm} 
%\kwd{NP oracle inequalities} 
%\kwd{plug-in approach} 
%\kwd{screening}
%\end{keyword}
%
%\end{frontmatter}

%%%%%%%%%%%%%%%%%%%%%%%%%%%%%%%%%%%%%%%%%%%%%%
%=============================================
\section{Introduction}
%=============================================
Classification plays an important role in many aspects of our society. 
In medical research, identifying pathogenically distinct tumor types is central to advances in cancer treatments \citep{Golub.99, alderton2014breast}. 
In cyber security, spam messages and  virus make automatic categorical decisions a necessity.  Binary classification is arguably the simplest and most important form of classification problems, and can serve as a building block for more complicated applications. We focus our attention on binary classification in this work. A few common notations are introduced 
to facilitate our discussion.  
Let $(X, Y)$ be a random pair where $X \in \cX \subset \R^d$ is a vector of features and $Y \in \{0,1\}$ indicates $X$'s class label.
A \emph{classifier} $\phi : \mathcal{X} \to \{0,1\}$ is a mapping from $\mathcal{X}$ to $\{0,1\}$ that assigns $X$ to one of the classes.  
A \emph{classification loss function} is defined to assign a ``cost" to each misclassified instance $\phi(X)\neq Y$,
and the \emph{classification error} is defined as the expectation of this loss function with respect to the joint distribution of $(X,Y)$. 
We will focus our discussion on the 0-1 loss function  $\1\{\phi(X)\neq Y\}$ throughout the paper, where $\1(\cdot)$ denotes the indicator function.
Denote by $\p$ and $\E$ the generic probability distribution and expectation, whose meaning depends on specific contexts. 
The classification error is
$
R(\phi)=\E \1\{\phi(X)\neq Y\} = \p\left\{\phi(X)\neq Y\right\}
$.
The law of total probability allows us to decompose it into a weighted average of type~I error $R_0(\phi)=\p \left\{\phi(X)\neq Y|Y=0\right\}$ and type~II error $R_1(\phi)=\p \left\{\phi(X)\neq Y|Y=1\right\}$ as
\begin{equation}\label{EQ:risk break down}
R(\phi) \,=\, \p(Y=0)R_0(\phi)+\p(Y=1)R_1(\phi)\,.
\end{equation}

With the advent of high-throughput technologies, 
classification tasks have experienced an exponential growth in the feature dimensions throughout the past decade. 
The fundamental challenge of ``high dimension, low sample size" has motivated the development of a plethora of classification algorithms for various applications. 
While dependencies among features are usually considered a crucial characteristic of the data \citep{Ackermann.Strimmer.2009}, and can effectively reduce classification errors under suitable models and relative data abundance \citep{Shao.Wang.ea.2011,Cai.Liu.2011,Fan.Feng.ea.2011,Mai.Zou.ea.2012, Witten.Tibshirani.2012}, 
independence rules, with their superb scalability, become a rule of thumb when the feature dimension grows faster than the sample size \citep{Hastie.Tibshirani.ea.2009,james2013introduction}.  
Despite Naive Bayes models' reputation of being ``simplistic" by ignoring all dependency structure among features, 
they lead to simple classifiers that have proven worthy on high-dimensional data with remarkably good performances in numerous real-life applications.  
Taking the classical model setting of two-class Gaussian with a common covariance matrix,  \citet{Bickel.Levina.2004} showed the superior performance of Naive Bayes models over (naive implementation of) the Fisher linear discriminant rule under broad conditions in high-dimensional settings. 
\cite{Fan.Fan.2008} further established the necessity of feature selection for high-dimensional classification problems by showing that even independence rules can be as poor as random guessing due to noise accumulation.
Featuring both independence rule and feature selection, the (sparse) Naive Bayes model remains a good choice for classification when the sample size is \emph{fairly limited}. 

%\textcolor{blue}{For example, 
%\citet{Bickel.Levina.2004} demonstrated the breakdown of Fisher discriminant rule which is based on single two class Gaussian assumption with common covariance.  and suggested Naive Bayes classifiers which assumes independence among features. 
%Modifications such as GRD, MCA, ZVD (\cite{Krzanowski.1995}) and SCRDA (\cite{Guo.Hastie.ea.2005}) were introduced to adapt the classic discriminant analysis methods to ``large-$p$-small-$n$" problems.   
%Other dimension reduction tools such as PLS (\cite{Wold.1966}) and PCA have also been recently adapted for high-dimensional classiﬁcation problems (\cite{Boulesteix-2004}, \cite{Zou.Hastie.ea.2006}). 
%\cite{Fan.Fan.2008} further established the necessity of feature selection for high-dimensional classification problems by showing that even independence rules can be as poor as random guessing due to noise accumulation in estimating population centroids in high-dimensional feature space.
%Procedures like FAIR (\cite{Fan.Fan.2008}) and ROAD (\cite{Fan.Feng.ea.2011}) were developd based on widely used variable selection techniques such as penalized quasi-likelihood function (\cite{Fan.Li.2006}) and two-sample $t$-test (\cite{Tibshirani.Hastie.ea.2002}, \cite{Nguyen-2002}) and provide effective solutions to the high-dimensional classification problems under the classic binary classification framework. }
 %\textcolor{red}{Copied from FANS:}

%%%%%%%%%%%%%%%%%%%%%%%%%
\subsection{Asymmetrical priorities on errors}
%%%%%%%%%%%%%%%%%%%%%%%%%%
Most existing binary classification methods target on the optimization of the overall risk \eqref{EQ:risk break down} and may fail to serve the purpose when users' relative priorities over type I/II errors  differ significantly from those implied by  the marginal probabilities of the two classes.
%\adc{differ significantly from those implied by the marginal probabilities of the two classes.}
A representative example of such scenario is the diagnosis of serious disease.
Let $1$ code the healthy class and $0$ code the diseased class.
Given that usually $$\p(Y=1) \gg \p(Y= 0)\,,$$
minimizing the overall risk \eqref{EQ:risk break down} might yield classifiers with small overall risk $R$ (as a result of small $R_1$) yet large $R_0$ --- a situation quite undesirable  in practice given  
flagging a healthy case incurs only extra cost of additional tests while failing to detect the disease endangers a life.
%%%%%%%%  Real data example --- penalized logistic regression / svm.}

The neuroblastoma dataset introduced by \cite{Oberthur.2006} provides a perfect illustration of such intuition.
The dataset contains gene expression profiles on $d=10707$ genes from 246 patients in a German neuroblastoma trial, among which 56 are high-risk (labeled as 0) and 190  are low-risk (labeled as 1).
We randomly selected 41 `$0$'s and 123 `$1$'s as our training sample (such that the proportion of `$0$'s is about the same as that in the entire dataset), and tested the resulting classifiers on the rest 15 `$0$'s and 67 `$1$'s.
The average error rates of \psns (to be proposed; implemented here at significance level 0.05), Gaussian Naive Bayes (nb), penalized logistic regression (pen-log),  and Support Vector Machine (svm) over 1000 random splits are summarized  in Table \ref{table:1}.
\begin{table}[h]
\caption{Average error rates over 1000 random splits for neuroblastoma dataset.\label{table:1} }
\begin{center}
\begin{tabular}{l r r r r r  }
%\hline
{Error Type} & \psn & nb &  pen-log & svm\\
\hline
type I\phantom{I} ($0$ as $1$) &  \underline{.038} &   .308 &  .529& .603\\
type II  ($1$ as $0$)  &  .761 &  .150 &  .103& .573 \\
%\hline
\end{tabular}
\end{center}
\end{table}
All procedures except {\psn} led to high type~I errors, and are thus considered unsatisfactory given the more severe consequences of missing a diseased instance than vice versa. 

%%%%% {Cost-sensitive learning}
One existing solution to asymmetric error control is \emph{cost-sensitive learning}, which assigns two different costs as weights of the type~I/II errors \citep{Elkan01,ZadLanAbe03}. %\textcolor{red}{should specifically comment on svm with asymmetric errors}
Despite many merits and practical values of this framework, limitations arise in applications when there is no consensus over how much costs to be assigned to each class, or more fundamentally, whether it is morally acceptable to assign costs in the first place. 
Also, when users have a specific target for type~I/II error control, cost-sensitive learning does not fit. Other methods aiming for small type~I error include the Asymmetric Support Vector Machine \citep{WuLinChenChen.2008}, and the $p$-value for classification \citep{DumbgenIglMunk.2005}. 
However,  the former has no theoretical guarantee on errors,  while the latter treats all classes as of equal importance.     %
%%%%%%%%%%%%%%%%%%%%%%%%%%%%%%%%%%%%%%%%%%%%%%%%%%%%%%%%%%%%%%%%%%%
\subsection{Neyman-Pearson (NP) paradigm and NP oracle inequalities}
%%%%%%%%%%%%%%%%%%%%%%%%%%%%%%%%%%%%%%%%%%%%%%%%%%%%%%%%%%%%%%%%%%%%%%
Neyman-Pearson (NP) paradigm was introduced as a novel statistical framework for targeted type~I/II error control.
Assume type~I error $R_0$ as the prioritized error type, 
this paradigm seeks to control $R_0$ under a user specified level $\alpha$ with $R_1$ as small  as possible.
The \emph{oracle} is thus
\begin{equation}
\label{eq::goal}
\phi^* \,\in\, \argmin_{R_0(\phi)\leq\alpha}R_1(\phi)\,,
\end{equation}
where the \emph{significance level} $\alpha$ reflects the level of conservativeness towards type I error.
Given $\phi^*$ is unattainable in the learning paradigm, the best within our capability is to construct a data dependent classifier $\hat{\phi}$ that mimics it.

Despite its practical importance, NP classification has not received much attention in the statistics and machine learning communities.   \citet{CanHowHus02} initiated the theoretical treatment of NP classification.  
Under the same framework, \citet{Sco05} and \citet{ScoNow05} derived several results for traditional statistical learning such as PAC bounds or oracle inequalities.  By combining type I and type II errors in sensible ways, \citet{Sco07} proposed a performance measure for NP classification. More recently,  \citet{BlaLeeSco10} developed a general solution to semi-supervised novelty detection by reducing it to NP classification.  
Other related works include \citet{CasChe03} and  \citet{HanCheSun08}.
A common issue with methods in this line of literature is that they all follow an empirical risk minimization (ERM) approach, and use some forms of relaxed empirical type I error constraint in the optimization program.
As a result, all type I errors can only be proven to satisfy some relaxed upper bound.
Take the framework set up by \citet{CanHowHus02} for example. Given $\varepsilon_0>0$, they proposed the program
$$
\min_{\phi\in\mathcal{H},\hat R_0(\phi)\leq \alpha + \varepsilon_0/2}\hat R_1(\phi)\,,
$$
where $\mathcal{H}$ is a set of classifiers with finite Vapnik-Chervonenkis  dimension, and $\hat R_0$, $\hat R_1$ are the empirical type I and type II errors respectively. It is shown that  with high probability, the solution $\hat{\phi}$ to the above program satisfies simultaneously: i) the type I error $R_0(\hat{\phi})$ is bounded from above by $\alpha + \varepsilon_0$, and ii) the type II error $R_1(\hat{\phi})$ is bounded from above by $R_1(\phi^*)+\varepsilon_1$ for some $\varepsilon_1>0$.

\citet{RigTon11} is a significant departure from the previous NP classification literature. 
This paper argues that a good classifier $\hat{\phi}$ under the NP paradigm  should respect the chosen significance level $\alpha$, rather than some relaxation of it. 
More precisely,  two \textbf{NP oracle inequalities} should be satisfied simultaneously with high probability:
\begin{itemize}
\item[(I)] the type~I error constraint is respected, i.e., $R_0(\hat{\phi})\leq\alpha$.
\item[(II)] the excess type~II error $R_1(\hat{\phi}) - R_1({\phi}^*)$ diminishes with explicit rates (w.r.t. sample size).
\end{itemize}
Recall that, for a classifier $\hat h$, the classical oracle inequality insists that with high probability %\textcolor{red}{cite some reference},
\begin{equation}
\label{classical}
\text{the excess risk $R(\hat h)- R(h^*)$ diminishes with explicit rates,}
\end{equation} 
where $h^*(x)=\1(\eta(x)\geq 1/2)$ is the Bayes classifier, in which $\eta(x)=\mathbb{E}[Y|X=x]=\p(Y=1|X=x)$ is the regression function of $Y$ on $X$ (see \citet{Koltchinskii.2008} and references within). 
The two NP oracle inequalities defined above can be thought of as a generalization of \eqref{classical} that provides a novel characterization of classifiers' theoretical performances under the NP paradigm. 

Using a more stringent empirical type I error constraint (than the level $\alpha$),  \citet{RigTon11} established NP oracle inequalities for its proposed classifiers under convex loss functions (as opposed to the indicator loss).  They also proved an interesting negative result: under the binary loss, ERM approaches (convexification or not) cannot guarantee diminishing excess type~II error as long as one insists type~I error of the proposed classifier be bounded from above by $\alpha$ with high probability. 
This negative result motivated  a plug-in approach to NP classification in \citet{Tong.2013}.

%The theoretical treatment of NP classification paradigm was initiated by \cite{CanHowHus02}. 
%Parallel to the classical binary classification,
%NP-type classifiers can be constructed in the following two ways:
%\begin{itemize}
%\item Empirical risk minimization (ERM) approach, which relaxes the constraint on type-I risk to constraining the empirical $\varphi$-type-I risk, and
%\item Plug-in approach, which replaces the unknown parts of a classifier defined based on the generally unknown ground truth by their estimates from the training data.
%\end{itemize}
%Some of the most recent work include
%\cite{RigTon11}, which proposed a computationally feasible classifier $\tilde{h}^\tau$ with $\varphi$-type-I error bounded from above by $\alpha$ with high probability and the excess $\varphi$-type II error converging to 0 with explicit rates, and  
%\cite{Tong.2012}, which proposed two plug-in classifiers based on the fundamental Neyman-Pearson Lemma and derived oracle inequalities that could be viewed as the finite sample version of the risk bounds.
%\cite{Sco05} proposed two families of performance measures for evaluating and comparing NP-type classifiers and presented general learning rules that satisfied performance guarantees with respect to these criteria. 
%

\subsection{Plug-in approaches}
Plug-in methods in classical binary classification have been well studied in the literature, where the usual plug-in target is the Bayes classifier $\1(\eta(x)\geq 1/2)$. Earlier works gave rise to pessimism of the plug-in approach to classification.  For example, under certain assumptions, \citet{Yang99} showed plug-in estimators cannot achieve excess risk with rates faster than $O(1/\sqrt{n})$, while direct methods can achieve rates up to $O(1/n)$ under \textit{margin assumption} \citep{MamTsy99, Tsy04, TsyGee05,TarGee06}.   However, it was shown in  \citet{Audibert05fastlearning} that plug-in classifiers $\1(\hat \eta_n \geq 1/2)$ based on local polynomial estimators can achieve rates faster than $O(1/n)$, with a smoothness condition on $\eta$ and  the margin assumption.
%\subsubsection{Neyman-Pearson Lemma}

The oracle classifier under the NP paradigm arises from its close connection to the Neyman-Pearson Lemma in statistical hypothesis testing. 
Hypothesis testing bears strong resemblance to binary classification if we assume the following  model. Let $P_1$ and $P_0$ be two \textit{known} probability distributions on $\cX \subset \R^d$. Assume that $Y\sim \text{Bern}(\zeta)$ for some $\zeta \in (0,1)$, and the conditional distribution of $X$ given $Y$ is $P_Y$.
%$$
%Y=\left\{
%\begin{array}{ll}
%1 & \textrm{with probability } \pi\,,\\
%0 & \textrm{with probability } 1-\pi\,.
%\end{array}\right.
%$$
Given such a model, the goal of statistical hypothesis testing is to determine if we should reject the null hypothesis that $X$ was generated from $P_0$. 
To this end, we construct a randomized test $\phi:\cX \to [0,1]$ that rejects the null with probability $\phi(X)$.  
Two types of errors arise: type~I error occurs when $P_0$ is rejected yet $X\sim P_0$, and type~II error occurs when $P_0$ is not rejected yet $X\sim P_1$. 
The Neyman-Pearson paradigm in hypothesis testing amounts to choosing $\phi$ that solves the following constrained optimization problem
$$
\text{maximize } \E[\phi(X)|Y=1]\,,
\text{ subject to }  \E[\phi(X)|Y= 0 ]\leq\alpha\,,
$$
where $\alpha \in (0,1)$ is the significance level of the test. A solution to this constrained optimization problem is called  \emph{a most powerful test} of level $\alpha$. The Neyman-Pearson Lemma gives mild sufficient conditions for the existence of such a test.

\begin{lem}[Neyman-Pearson Lemma]\label{lemma:NP}
Let $P_1$ and $P_0$ be two probability measures with densities $p
$ and $q$ respectively, and denote the density ratio as $r(x)=p(x)/q(x)$.
For a given significance level $\alpha$, let $C_{\alpha}$ be such that
$P_0\{r(X)>C_{\alpha}\}\leq\alpha$ and $P_0\{r(X)\geq C_{\alpha}\}\geq\alpha$.  
Then,
the most powerful test of level $\alpha$ is
\begin{equation*}
\phi^*(X)=\left\{
 \begin{array}{ll}
     1 & \text{if $\,\,r(X)>C_{\alpha}$}\,,\\
     0 & \text{if $\,\,r(X)<C_{\alpha}$}\,,\\
     \frac{\alpha-P_0\{r(X)>C_{\alpha}\}}{P_0\{r(X)=C_{\alpha}\}} & \text{if $\,\,r(X)=C_{\alpha}$}\,.
   \end{array}      \right.
\end{equation*}
\end{lem}
Under mild continuity assumption, we take the NP \emph{oracle} 
\begin{align}\label{eq::oracle}
\phi^*(x) \,=\, \phi^*_{\alpha}(x) \,=\, \1\{p(x)/q(x)\geq C_{\alpha}\} \,=\, \1\{r(x)\geq C_{\alpha}\}\,.
\end{align}
as our plug-in target for NP classification.
With kernel density estimates $\hat p$, $\hat q$, and a proper estimate of the threshold level $\widehat C_{\alpha}$, \cite{Tong.2013} constructed a plug-in classifier $\1\{\hat p(x)/\hat q(x)\geq \widehat C_{\alpha}\}$ that satisfies both NP oracle inequalities with high probability when the dimensionality is small, leaving the high-dimensional case an unchartered territory.

%%%%%%%%%%%%%%%
\subsection{Contribution}
%%%%%%%%%%%%%%%
In the big data era, 
NP classification framework faces the same curse of dimensionality as its classical counterpart. 
Despite its wide potential applications, this paper is the \emph{first attempt} to construct performance-guaranteed classifiers under the NP paradigm in high-dimensional settings. 
Based on the Neyman-Pearson Lemma, we employ Naive Bayes models and propose a computationally feasible plug-in approach to construct classifiers that satisfy the NP oracle inequalities. 
We also improve the \textit{detection condition}, a critical theoretical assumption first introduced in \cite{Tong.2013}, for effective threshold level estimation that grounds the good NP properties of these classifiers. Necessity of the new detection condition is also discussed.   
%We now revisit the neuroblastoma dataset \citep{Oberthur.2006} and look at the first column in Table \ref{table:1}. It is clear that \psn, one of our proposed classifiers, works well in controlling type I error under the significance level $\alpha = 0.05$, while the other popular classification algorithms fail to do so.
Note that classifiers proposed in this work are not  straightforward extensions of \cite{Tong.2013}: 
kernel density estimation is now applied in combination with feature selection,  and the threshold level is estimated in a more precise way by order statistics that require only moderate sample size --- while \citet{Tong.2013} resorted to the Vapnik-Chervonenkis theory and required sample size much bigger than what is available in most high-dimensional applications.

%The SNNB procedure, as claimed, keeps the type I error rates under the pre-specified level in this high-dimensional example. 
%To the best of our knowledge, it is the first work in the setting of high-dimensional Neyman-Pearson binary classification. 

The rest of the paper is organized as follows. Two screening based plug-in NP-type classifiers are presented in Section \ref{sec::methods}, where theoretical properties are also discussed. Performance of the proposed classifiers is demonstrated in Section \ref{sec::numeric} by both simulation studies and real data analysis. 
We conclude in Section \ref{sec::discussion} with a short discussion. 
The technical proofs are relegated to the Appendix.

\section{Methods\label{sec::methods}}
In this section, we first introduce several notations and definitions, with a focus on the \textit{detection condition}.  Then we present the plug-in procedure, together with its theoretical properties.  

\subsection{Notations and definitions\label{sec::notations-definitions}}
%\textcolor{red}{As we discussed before, change to $1-3\delta$ to $1-\delta_1 + ...$ in the main text}
We introduce here several notations adapted from \citet{Audibert05fastlearning}.
For $\beta>0$,  denote by $\lfloor\beta\rfloor$ the largest integer strictly less than $\beta$.  For any $x, x'\in \mathbb{R}$ and any $\lfloor\beta\rfloor$ times continuously differentiable real-valued function $g(\cdot)$ on $\mathbb{R}$, we denote by $g_x$ its Taylor polynomial of degree $\lfloor\beta\rfloor$ at point $x$.
For $L>0$,  the $(\beta, L, [-1,1])$-H\"{o}lder class of functions, denoted by $\Sigma(\beta, L, [-1,1])$, is the set of functions $g:[-1,1]\rightarrow \mathbb{R}$ that are $\lfloor\beta\rfloor$ times continuously differentiable and satisfy, for any $x,x'\in[-1,1]$, the inequality $|g(x')-g_x(x')|\leq L |x-x'|^{\beta}.$
The $(\beta, L, [-1,1])$-H\"{o}lder class of density is defined as
$$
\mathcal{P}_{\Sigma}(\beta, L, [-1,1])=\left\{f\,:\,f\geq0, \int f = 1, f\in\Sigma(\beta, L, [-1,1])\right\}\,.
$$

We will use $\beta$-valid kernels (kernels of order $\beta$, \citet{Tsy09}) for all the kernel estimation throughout the theoretical discussion, the definition of which is as follows.
\begin{defin}\label{DEF:BETA kernel}
Let $K(\cdot)$ be a real-valued function on $\mathbb{R}$ with support $[-1,1]$.  The function $K(\cdot)$ is a $\beta$-valid kernel if it satisfies $\int K = 1$, $\int |K|^v <\infty$ for any $v \geq 1$, $\int |t|^{\beta}|K(t)|dt<\infty$, and in the case $\lfloor\beta\rfloor\geq 1$, it satisfies $\int t^l K(t)dt = 0$ for any $l\in\mathbb{N}$ such that $1\leq l\leq \lfloor\beta\rfloor$.
\end{defin}
We assume that all the $\beta$-valid kernels considered in the theoretical part of this paper are constructed from Legendre polynomials, and are thus Lipschitz and bounded, satisfying the kernel conditions for the important technical Lemma \ref{lemma::A1_1-dim}.
\begin{defin}[margin assumption]\label{Def: margin}
A function $f(\cdot)$ is said to satisfy margin assumption of order $\bar\gamma$ with respect to probability distribution $P$ at the level $C^{*}$ if there exists a positive constant $M_0$, such that for any $\delta\geq0$,
$$
P\{|f(X)-C^*|\leq \delta\} \,\leq\, M_0\delta^{\bar\gamma}\,.
$$
\end{defin}
This assumption  was first introduced in \citet{Polonik95}. In the classical binary classification framework,  \cite {MamTsy99} proposed a similar condition named ``margin condition"
 by requiring  most data to be away from the optimal decision boundary. 
In the classical classification paradigm, definition \ref{Def: margin} reduces to the ``margin condition" by taking $f = \eta$ and $C^*=1/2$, with $\{x: |f(x)-C^*| = 0\} = \{x: \eta(x)=1/2\}$ giving the decision boundary of the Bayes classifier.  
On the other hand,  unlike the classical paradigm where the optimal threshold level is known and does not need an estimate, the optimal threshold level $C_{\alpha}$ in the NP paradigm is unknown and needs to be estimated, suggesting  the necessity of having sufficient data around the decision boundary to detect it well.  This concern motivated the following condition improved from \citet{Tong.2013}.

\begin{defin}[detection condition] 
\label{def::detection}
A function $f(\cdot)$ is said to satisfy detection condition  of order $\uderbar \gamma$ with respect to $P$ (i.e., $X\sim P$) at level $(C^*,\delta^*)$ 
if there exists a positive constant $M_1$, such that for any $\delta \in (0,\delta^*)$,
$$P\{C^{*} \leq f(X) \leq C^{*} + \delta \} \,\geq\, M_1 \delta^{\uderbar\gamma}\,.$$
\end{defin}

A detection condition works as an opposite force to the margin assumption, and is basically an assumption on the lower bound of probability. 
Though we take here a power function as the lower bound, so that it is simple and aesthetically similar to the margin assumption, any increasing $u(\cdot)$ on $R^+$ with $\lim_{x\rightarrow 0+} g(x) = 0$ should be able to serve the purpose.  
The version of detection condition we would use to establish the NP inequalities for the (to be) proposed classifiers takes $f = r$, $C^* = C_{\alpha}$, and $P = P_0$ (recall that $P_0$ is the conditional distribution of $X$ given $Y=0$).  

Now we argue why such a condition is \textit{necessary} to achieve the NP oracle inequalities.   
Consider the simpler case where the density ratio $r$ is known, and we only need a proper estimate of the threshold level $\widehat{C}_\alpha$. If there is nothing like the detection condition (Definition \ref{def::detection} involves a power function, but the idea is just to have any kind of lower bound), we would have, for some $\delta>0$, 
%%%%%%
\begin{align}\label{eq::detection-not-hold}
P_0\{ C_{\alpha} \leq r(X) \leq C_{\alpha}+\delta\} \,=\, 0\,.
\end{align}  
%There are two possible cases here. The analysis of the second case $P_0\{ C_{\alpha}-\delta \leq r(X) \leq C_{\alpha}\}$ is more involved and will be deferred to Appendix \ref{sec::assumption3 and detection condition}. Here, we discuss the implication of the first case $P_0\{ C_{\alpha} \leq r(X) \leq C_{\alpha}+\delta\} \,=0$. 
In getting the threshold estimate $\widehat{C}_{\alpha}$ of $\hat \phi (x) = \1\{r(x)\geq \widehat{C}_{\alpha}\}$, we can not distinguish any threshold level between $C_{\alpha}$ and $C_{\alpha}+\delta$. In particular, it is possible that 
%%%%%%
$$
\widehat{C}_{\alpha} > C_{\alpha} + \delta/2\,.
$$
But then the excess type II error is bounded from below as follows

$$ 
R_1(\hat \phi) - R_1(\phi^*) =  P_1\{C_{\alpha} < r(X) < \widehat{C}_{\alpha}\} > P_1\{C_{\alpha} < r(X) < C_{\alpha}+\delta/2\}\,,
$$
where the last quantity can be positive. Therefore, the second NP oracle inequality (diminishing excess type II error) does not hold for $\hat \phi$. 
Since some detection condition is necessary in this simpler case, it is certainly necessary in our real setup. 

Note that Definition \ref{def::detection} is a significant improvement of the detection condition formulated in \cite{Tong.2013}, which requires 
$$P\{C^{*}-\delta  \leq f(X) \leq C^{*} \} \wedge P\{C^{*} \leq f(X) \leq C^{*} + \delta \} \,\geq\, M_1 \delta^{\uderbar\gamma}\,.$$
We are able to drop the lower bound for the first piece due to an improved layout of the proofs. Intuitively, our new detection condition ensures an upper bound on $\widehat C_{\alpha}$. But we do not need an extra condition to get a lower bound of $\widehat C_{\alpha}$, because of the type I error bound requirement (see the proof of Proposition \ref{prop::R1} for details).  
 
%The above condition generalizes the margin assumption in \citet{Audibert05fastlearning}, which deals with classical binary classification, and where $C^*=1/2$ and $p$ is limited to the regression function $\eta$. 

\subsection{Neyman-Pearson plug-in procedure}
Suppose the sampling scheme is fixed as follows. 
\begin{assumption}\label{assumption::independence-split}
Assume the training sample contains 
$n$ i.i.d. observations $\mathcal{S}^1=\{U_1,\cdots,U_n\}$ from class 1 with density $p$, and 
 $m$ i.i.d. observations $\mathcal{S}^0=\{V_1,\cdots,V_m\}$  from class 0 with density $q$.
Given fixed $n_1$, $n_2$, $m_1$, $m_2$ and $m_3$ such that $n_1+n_2 = n$, $m_1+m_2+m_3=m$,
we further decompose $\mathcal{S}^1$ and $\mathcal{S}^0$ into independent subsamples as:
$\mathcal{S}^1 = \mathcal{S}^1_1 \cup \mathcal{S}^1_2$, and  $\mathcal{S}^0 = \mathcal{S}^0_1 \cup \mathcal{S}^0_2 \cup \mathcal{S}^0_3$, where $|\mathcal{S}^1_1| = n_1$, $|\mathcal{S}^1_2| = n_2$, $|\mathcal{S}^0_1| = m_1$, $|\mathcal{S}^0_2| = m_2$, $|\mathcal{S}^0_3| = m_3$.  
\end{assumption}
The sample splitting idea has been considered in the literature, such as in \citet{Meinshausen.Buhlmann.2010} and \citet{robins2006adaptive}. Given these samples, we introduce the following plug-in procedure.
\begin{defin}\label{pro::np-plug-in}
 {\bf Neyman-Pearson plug-in procedure}
\begin{itemize}
\item[\underline{Step 1}] Use $\mathcal{S}^1_1$, $\mathcal{S}^1_2$, $\mathcal{S}^0_1$, and $\mathcal{S}^0_2$ to construct a density ratio estimate $\hat r$. The specific use of each subsample will be introduced in Section \ref{sec::density-ratio-estimate}.
\item[\underline{Step 2}] Given $\hat\Lr,$ choose a threshold estimate $\Ca$ from the set 
${\hat \Lr}(\mathcal{S}^0_3) = \{\hat\Lr(V_{i+m_1+m_2}) \}_{i=1}^{m_3}$.
\end{itemize}
\end{defin}

Denote by ${\hat \Lr}_{(k)}(\mathcal{S}^0_3)$ the $k$-th order statistic of ${\hat \Lr}(\mathcal{S}^0_3)$, $k \in \{1,\cdots,m_3\}$. 
The corresponding plug-in classifier by setting $\Ca={\hat \Lr}_{(k)}(\mathcal{S}^0_3)$ is
\begin{eqnarray}\label{eq:threshold-generic}
\label{eq::psiK}
\hat{\phi}_k (x) =  \1\{ \hat \Lr(x) \geq {\hat \Lr}_{(k)}(\mathcal{S}^0_3)\}\,.
\end{eqnarray}
\noindent A generic procedure for choosing the optimal $k$ will be given in Section \ref{sec::threshold-estimate}.  %Estimation of $r$ (Step 1) will be studied in Section \ref{sec::density-ratio-estimate}.

\subsection{Threshold estimate $\widehat C_{\alpha}$\label{sec::threshold-estimate}}
For any arbitrary density ratio estimate $\hat r$, 
we employ a proper order statistic ${\hat \Lr}_{(k)}(\mathcal{S}^0_3)$ to estimate the threshold $C_{\alpha}$, and establish a probabilistic upper bound for the type I error of $\hat{\phi}_k$ for each $k \in \{1,\cdots,m_3\}$.   
%Let denote the \textsc{cdf} of random variable $\Lr(X),$ where $X \sim p_0,$ 
%the following proposition gives a probabilistic upper bound of $R_0(\psiK)$ as $\tS_0$ varies according to $P_{0^{\tn}}.$

\begin{prop}
\label{prop::general delta}
For any arbitrary density ratio estimate $\hat r$, 
let $\hat \phi_k(x)=\1\{ \hat \Lr(x) \geq {\hat \Lr}_{(k)}(\mathcal{S}^0_3)\}$. 
It holds for any $\delta \in (0, 1)$  and $k\in \{1,\cdots, m_3\}$ that 
\begin{eqnarray}
\label{eq::Bin-bound}
\p\{ R_0(\hat{\phi} _k) > \delta\}
 \,\leq\, \text{Beta.cdf}_{k, \,m_3+1-k}\left( 1-\delta \right)\,,
\end{eqnarray}
where $\text{Beta.cdf}_{k, \,m_3+1-k}(\cdot)$ is the \textsc{cdf} of Beta$(k, \,m_3+1-k)$.
The inequality becomes equality when $\Fr(t)=P_0\{ \hat \Lr(X) \leq t \}$ is continuous almost surely.
\end{prop}
In view of the above proposition, 
a sufficient condition for the classifier $\hat{\phi}_k$ to satisfy NP Oracle Inequality (I) at tolerance level $\delta_3 \in (0,1)$ is thus
\begin{eqnarray}
\label{cond::betaCDF}
 \text{Beta.cdf}_{k, \,m_3+1-k}\left( 1-\alpha \right) \,\leq\, \delta_3\,.
\end{eqnarray}
Despite the potential tightness of \eqref{eq::Bin-bound}, we are not able to derive an explicit formula for the minimum $k$ that satisfies \eqref{cond::betaCDF}.
To get an explicit choice for $k$, we resort to concentration inequalities for an alternative.  

\begin{prop}
\label{prop::general k}
For any arbitrary density ratio estimate $\hat r$, 
let $\hat \phi_k(x)=\1\{ \hat \Lr(x) \geq {\hat \Lr}_{(k)}(\mathcal{S}^0_3)\}$. 
It holds for any $\delta_3 \in (0, 1)$  and $k\in \{1,\cdots, m_3\}$  that  
\begin{eqnarray}
\label{eq::Chebyshev-bound}
\p\{ R_0(\hat{\phi} _k) > g(\delta_3, m_3, k) \}
\,\leq \, \delta_3\,,
\end{eqnarray}
where
\begin{equation}
\label{eq::g}
g(\delta_3, m_3, k) = 
\frac{ m_3+1-k}{m_3+1} +  
\sqrt{\frac{k(m_3+1-k)}{\delta_3(m_3+2)(m_3+1)^2}}\,.
\end{equation}
\end{prop}

\medskip
Let
$
\mathcal{K} = \mathcal{K}(\alpha,\delta_3,m_3) 
= \{k \in \{1,\cdots,m_3\}:
g(\delta_3, m_3, k) \leq \alpha \}
$. 
Proposition  \ref{prop::general k} implies that  $k\in \mathcal{K}(\alpha,\delta_3,m_3) $ is a sufficient condition for the classifier $\hat{\phi}_k$ to satisfy NP Oracle Inequality (I). The next step is to characterize $\mathcal{K}$  and  choose some $k\in\mathcal{K}$, so that $\hat \phi_k$  has small excess type II error.  Clearly, we would like to find the smallest element in $\mathcal{K}$.

%%%%%%%%%%%%%%%%%%%%%%%%%%%%%%%%%%%%%%%%%%%%%%%%%%%%%%%%%%%%%%
%      PROP-OF-Kmin      PROP-OF-Kmin      PROP-OF-Kmin 
%      PROP-OF-Kmin      PROP-OF-Kmin      PROP-OF-Kmin 
%%%%%%%%%%%%%%%%%%%%%%%%%%%%%%%%%%%%%%%%%%%%%%%%%%%%%%%%%%%%%%
\begin{prop}
\label{prop::kmin}
The minimum $k \in \{1,\cdots,m_3+1\}$ that satisfies $g(\delta_3,m_3,k) \leq \alpha$ is 
\begin{equation}
\label{eq::kmin}
k_{\min}(\alpha,\delta_3,m_3) \,=\, \left\lceil (m_3+1)A_{\alpha,\delta_3}(m_3) \right\rceil,
\end{equation}
where $\lceil z\rceil$ denotes the smallest integer larger than or equal to $z$, and  
\begin{equation*}
\label{eq::A}
A_{\alpha,\delta_3}(m_3) = \frac{1+2\delta_3(m_3+2)(1-\alpha) + \sqrt{1+4\delta_3(1-\alpha)\alpha(m_3+2)}}
{2\left\{\delta_3(m_3+2)+1 \right\}}\,.
\end{equation*}
Moreover,
\begin{enumerate}
\item $A_{\alpha,\delta_3}(m_3) \in (1-\alpha,1)$.
\item ${\hat  r}_{(k_{\min}(\alpha, \delta_3, m_3))}(\mathcal{S}_3^0)$ is asymptotically the empirical $(1-\alpha)$-th quantile of  $\Fr$ in the sense that  
\begin{equation*}
\label{eq::limit}
\lim_{m_3\to\infty}\frac{k_{\min}(\alpha,\delta_3,m_3)}{m_3}\,=\, \lim_{m_3\to\infty} A_{\alpha,\delta_3}(m_3)\,=\,1 - \alpha\,.
\end{equation*}

\item 
For any $m_3 \geq 4/(\alpha\delta_3)$, we have
$k_{\min}(\alpha,\delta_3,m_3) \leq m_3$, and thus
$$\mathcal{K}(\alpha,\delta_3,m_3) \,=\, \left\{ k_{\min}(\alpha,\delta_3,m_3),k_{\min}(\alpha,\delta_3,m_3)+1,\ldots,m_3\right\}.$$ 
\end{enumerate}
\end{prop}

\noindent Introduce shorthand notations ${k_{\min}} = {k_{\min}}(\alpha, \delta_3, m_3)$, $\hat r_{(k)} =\hat r_{(k)}(\mathcal{S}_3^0)$, and $\widehat C_{\alpha} =  \hat r_{(\min\{k_{\min},m_3\})}$. 
We will take
\begin{align}\label{eq::phi-hat}
\hat{\phi} (x) \,=\, \1\{ \hat r(x) \geq \widehat C_{\alpha}\} \,=\,  \left\{
\begin{array}{ll}
\1\{\hat r(x) \geq \hat r_{(k_{\min})}\}\,,&\text{if $\,\,k_{\min} \leq m_3$}\,,\\
\1\{\hat r(x) \geq \hat r_{(m_3)}\}\,,& \text{if $\,\,k_{\min} = m_3+1$}
\end{array}\right.
\end{align}
as the default NP plug-in classifier for any arbitrary $\hat r$.
An alternative threshold estimate that also guarantees type I error bound is derived in the Appendix \ref{sec::alternative threshold}. 
Assume $m_3\geq 4/(\alpha\delta_3)$ for the rest of the theoretical discussion.
It follows from Proposition \ref{prop::kmin} that $k_{\min} \leq m_3$, and thus $\widehat C_{\alpha} =  \hat r_{(k_{\min})}$, $\hat\phi = {\hat \phi}_{(k_{\min})}$ with guaranteed type I error control.

%\adc{Do we want to unify "cutoff" and "threshold" throughout the paper? \color{red} All ``cutoff"s that refers to $\widehat{C}$ are unified to ``threshold / threshold level". One ``cutoff" left, refers to the $\tau$ in screening}
\begin{rem}
\label{rem::quantileIntuition}
Note that 
$\lim_{m_3\to\infty}{k_{\min}}/{\lceil m_3(1-\alpha)\rceil} = 1$.
Thus, choosing the $k_{\min}$-th order statistic of $\hat r(\mathcal{S}^0_3)$ as the threshold can be viewed as a modification to the classical approach of estimating the $1-\alpha$ quantile of $F_{0,\hat r}$ by the $\lceil m_3(1-\alpha)\rceil$-th order statistic of ${\hat \Lr}(\mathcal{S}^0_3)$.
Recall that the oracle $C_\alpha$ is actually the $1-\alpha$ quantile of distribution $F_{0,r},$
so the intuition is that $\widehat C_{\alpha}$ is asymptotically (when $m_3\rightarrow \infty$) equivalent to the $1-\alpha$ quantile of $F_{0,\hat \Lr},$ which in turn converges (when $n_1, n_2, m_1, m_2 \rightarrow \infty$) to $C_\alpha$ as the $1-\alpha$ quantile of $F_{0, \Lr}$ under moderate conditions.
\end{rem}

\begin{lem}
\label{prop::R0}

Let  $\alpha, \delta_3\in(0,1)$. In addition to Assumption \ref{assumption::independence-split}, suppose $\hat \Lr$ be such that $F_{0,\hat \Lr}$ is continuous almost surely.
Then for any $\delta_4 \in (0,1)$ and $m_3  \geq 4/(\alpha\delta_3)$,
the distance between $R_0( \hat{\phi})$ ($\hat{\phi}$ as defined in \eqref{eq::phi-hat})  and $R_0(\phiStar)$ can be bounded as 
\begin{eqnarray*}
\p\{
|R_0( \hat{\phi} ) 
- R_0( \phiStar) | > \xi_{\alpha, \delta_3,m_3}(\delta_4) \}
\,\leq\, \delta_4\,,
\end{eqnarray*}
where 
\begin{align}
\label{eq::xi}
\xi_{\alpha, \delta_3,m_3}(\delta_4) = \sqrt{\frac{\kHat(m_3+1-\kHat)}{(m_3+2)(m_3+1)^2\delta_4}} + A_{\alpha,\delta_3}(m_3) - (1-\alpha) + \frac{1}{m_3+1}\,.
\end{align}
If $m_3 \geq \max(\delta_3^{-2}, \delta_4^{-2})$, we have 
$
\xi_{\alpha, \delta_3,m_3}(\delta_4) \leq  ({5}/{2}){m_3^{-1/4}}.
$
\end{lem}

\begin{prop}
\label{prop::R1}
Let $\alpha, \delta_3, \delta_4 \in (0,1)$. In addition to assumptions of Lemma \ref{prop::R0}, assume that the density ratio $\Lr$ satisfies the margin assumption of order $\bar\gamma$ at level $C_{\alpha}$ (with constant $M_0$) and detection condition of order $\uderbar \gamma$ at 
level $(C_{\alpha}, \delta^*)$ (with constant $M_1$), both with respect to distribution  $P_0$.  
%\vskip -5pt
\noindent
If $m_3 \geq \max\{ 4/{(\alpha\delta_3)}, \delta_3^{-2}, \delta_4^{-2},(\frac{2}{5}M_1{\delta^*}^{\uderbar\gamma})^{-4}\}$,  the excess type II error of the classifier $\hat{\phi}$ defined in \eqref{eq::phi-hat} satisfies with probability at least $1-\delta_3-\delta_4$,
\begin{align*}
&R_1(\hat{\phi}) - R_1({\phi}^*)\\
&\leq\, 
2M_0 \left[\left\{\frac{|R_0( \hat{\phi}) - R_0( \phi^*)|}{M_1}\right\}^{1/\uderbar{\gamma}} + 2  \| \hat \Lr - \Lr \|_{\infty} \right]^{1 + \bar\gamma} 
+ C_{\alpha} |R_0( \hat{\phi}) - R_0( \phi^*)|\\
%&\leq& 
%2M_0 \left[\left(\frac{\xi_{\alpha,\delta_3,m_3}(\delta_4)}{M_1}\right)^{1/\uderbar{\gamma}} + 2  \| \hat \Lr - \Lr \|_{\infty} \right]^{1 + \bar\gamma} 
%+ C_{\alpha} \cdot \xi_{\alpha,\delta_3,m_3}(\delta_4)\\
&\leq\,
2M_0 \left[\left(\frac{2}{5}m_3^{1/4}M_1\right)^{-1/\uderbar{\gamma}} + 2  \| \hat \Lr - \Lr \|_{\infty} \right]^{1 + \bar\gamma} 
+ C_{\alpha} \left(\frac{2}{5} m_3^{1/4}\right)^{-1}\,.
\end{align*}
\end{prop}
Given the above proposition, we can control the excess type II error as long as the uniform deviation of density ratio estimate $\|\hat r-r\|_{\infty}$ is controlled. 
In the following subsection, we will introduce estimates $\hat r$ and provide bounds for $\|\hat r-r\|_{\infty}$.

\subsection{Density ratio estimate $\hat r$\label{sec::density-ratio-estimate}}
Denote the marginal densities of class 1 and 0 as $p_j$ and $q_j$ ($j=1,\cdots,d$) respectively, 
Naive Bayes models for the density ratio take the form
\begin{align*}
r(x) = \prod_{j=1}^d \frac{p_{j}(x_j)}{q_{j}(x_j)}\,, \quad\text{where $x_j$ is the $j$-th component of $x$}\,.
\end{align*}

The subsamples $\mathcal{S}^1_1=\{U_i\}_{i={1}}^{n_1}$, $\mathcal{S}^1_2=\{U_{i+n_1}\}_{i={1}}^{n_2}$, $\mathcal{S}^0_1=\{V_i\}_{i={1}}^{m_1}$ and $\mathcal{S}^0_2=\{V_{i+m_1}\}_{i={1}}^{m_2}$ are used to construct (nonparametric/parametric) estimates of $p_j$ and $q_j$  for $j = 1,\cdots,d$.

\vskip 6pt
\noindent \textbf{Nonparametric estimate of the density ratio}.
For marginal densities $p_j$ and $q_j$, we apply kernel estimates 
$\,\,\hat p_{j}(x_j) = \{(n_1+n_2) h_1\}^{-1}\sum_{i=1}^{n_1+n_2} K\left(\frac{U_{i,j} - x_j}{h_1}\right)$, and  
$\,\,\hat q_{j}(x_j) =  \{(m_1+m_2) h_0\}^{-1}\sum_{i=1}^{m_1+m_2} K\left(\frac{V_{i,j} - x_j}{h_0}\right)$,
where 
$K(\cdot)$ is the kernel function,
$h_1,h_0$ are the bandwidths, 
and $V_{i,j}$ and $U_{i,j}$ denote the $j$-th component of $V_{i}$ and $U_{i}$ respectively. 
The resulting nonparametric estimate  is
\begin{eqnarray}
\label{eq::rHatI}
\hat \Lr_{\text{N}}( x) 
= \prod_{j=1}^d \frac{\hat p_{j}(x_j)}{\hat q_{j}(x_j)}\,.
\end{eqnarray}

\noindent \textbf{Parametric estimate of the density ratio}.
Assume the two-class Gaussian model 
$
X|Y=0 \sim \n({\mu}^0,\Sigma) \text{ and } X|Y=1 \sim \n({\mu}^1,\Sigma ),
$ where $\Sigma=\mbox{diag}(\sigma_1^2,\cdots,\sigma_d^2)$. 
We estimate $\mu^0$, $\mu^1$ and $\Sigma$ using their sample versions $\hat \mu^0$, $\hat \mu^1$ and $\hat \Sigma$.
Under this model, the density ratio function is given by
\begin{eqnarray*}
\label{eq::rI}
\Lr_{\text{P}}( x)
\,=\, \exp\left\{\left(\mu^1-\mu^0\right)' \Sigma^{-1}x + \frac{1}{2}   (\mu^0)^\prime\Sigma^{-1} \mu^0 -  \frac{1}{2} (\mu^1)^\prime\Sigma^{-1}\mu^1 \right\}\,,
\end{eqnarray*}
and the corresponding parametric estimate is
\begin{equation}
\label{eq::rIhatVec}
\hat \Lr_{\text{P}}( x) 
\,=\, \exp\left\{\left(\hat\mu^1-\hat\mu^0\right)' \hat\Sigma^{-1}x + \frac{1}{2}  (\hat\mu^0)^\prime\hat\Sigma^{-1} \hat\mu^0 - \frac{1}{2}  (\hat\mu^1)^\prime\hat\Sigma^{-1}\hat\mu^1 \right\}\,.
\end{equation}
\subsection{Screening-based density ratio estimate and plug-in procedures\label{sec::screen-density-ratio-estimate}} 

For ``high dimension, low sample size" applications, complex models that take into account all features usually fail; even Naive Bayes models that ignore feature dependency might lead to poor performance due to noise accumulation \citep{Fan.Fan.2008}.  A common solution in these scenarios is to first study marginal relations between the response and each of the features \citep{Fan.Lv.2008, li2012feature}.  By selecting the most important individual features, we greatly reduce the model size, and other models can be applied after this screening step.   We now introduce screening based variants of $\hat r_{\text{N}}$ and $\hat  r_{\text{P}}$. Let  $F^{0}_j$ and  $F^{1}_j$ denote the \textsc{cdf}s of $q_j$ and $p_j$ respectively, for $j=1,\cdots,d$. 
  Step 1 of Procedure \ref{pro::np-plug-in} introduced in Section \ref{sec::notations-definitions} is now decomposed into a screening substep and an estimation substep.

%Denote the index set of signal dimensions and its size  by $\cA$ $(\subset \{1,\cdots,d\})$ and $d_s \equiv |\cA|,$ respectively. 
%Under sparsity independence assumption,  
%$d_s \ll d,$ and an ideal Na\"ive Bayes estimator of $r$ takes the form:
%$$
%\hat \Lr_{ideal}(\bm z) =  \prod_{j\in\cA} \frac{\hat p_{1,j}(z_j)}{\hat p_{0,j}(z_j)}, \quad
%%\widehat{\log\Lr_{ideal}}(\bm z)  =  \sum_{j\in\cA} \log\frac{\hat p_{1,j}(z_j)}{\hat p_{0,j}(z_j)}.
%$$
%Note that $\hat \Lr_{ideal},$ as suggested by its subscript, is a too ideal to be realized estimator in real settings due to the inherent challenge of having full information about  $\cA.$
%Yet statistical methods enable us to approximate it by
%$$
%\hat \Lr^s(\bm z) =  \prod_{j\in\cAhat} \frac{\hat p_{1,j}(z_j)}{\hat p_{0,j}(z_j)}, \quad
%%\widehat{\log\Lr^s}(\bm z) =  \sum_{j\in\cAhat} \log\frac{\hat p_{1,j}(z_j)}{\hat p_{0,j}(z_j)}.
%$$
%where $\cAhat$ is an estimate of $\cA.$
%Given the sparsity of true signals, the process of selecting elements of $\cAhat$ from $\{1,\cdots,d\}$ is essentially a ``screening."

%%%%%%%%%%%%%%%%%%%%%%%%%%%%%%%%%%%%%%%%%%%%
% t-screening
%%%%%%%%%%%%%%%%%%%%%%%%%%%%%%%%%%%%%%%%%%%%
% =================================================================
%\subsection{Screening with exact recovery}
% ==================================================================
%%%%%%%%%%%%%%%%%%%%%%%%%%%%%%%%%%%%%%%%%%%%%

\vskip 6pt

\noindent\textbf{\underline{N}onparametric \underline{S}creening-based \underline{N}P \underline{N}aive Bayes (\nsn) classifier}
\begin{description}
\item [\underline{Step 1.1}]  Select features using  $\mathcal{S}^0_1$ and $\mathcal{S}^1_1$ as follows:  
\begin{equation}
\label{eq::Atau}
\widehat{\mathcal{A}}_{\tau} 
\,=\, \left\{ 1 \leq j \leq d:
\| \hat{F}^{0}_{j} -\hat{F}^{1}_{j} \| _{\infty} \geq \tau\right\},
\end{equation}
where $\tau>0$ is some  threshold level, and
\begin{equation}
\label{eq::ecdf}
\hat{F}^{0}_{j}(x_j) \,=\, \frac{1}{m_1}\sum_{i=1}^{m_1} \1(V_{i,j} \leq x_j)\,,\,\,\,
\hat{F}^{1}_{j}(x_j)\,=\,\frac{1}{n_1}\sum_{i=1}^{n_1} \1(U_{i,j} \leq x_j)
\end{equation}
are the empirical \textsc{cdf}s.
%%%%%%%%%%%%%%%%%%%%%
\item [\underline{Step 1.2}] Use $\mathcal{S}^0_2$ and $\mathcal{S}^1_2$  to construct kernel estimates of $q_{j}$ and $p_{j}$ for $j \in \widehat{\mathcal{A}}_{\tau}$. The density ratio estimate is given by 
$$
\hat{\Lr}^S_{\text{N}}(x) \,=\, \prod_{j \in \widehat{\cA}_{\tau}} \frac{\hat p_{j}(x_j)}{\hat q_{j}(x_j)}\,.
$$
\item [\underline{Step 2}] Given $\hat{\Lr}^S_{\text{N}}$, use $\mathcal{S}^0_3$ to get a threshold estimate $(\hat{\Lr}^S_{\text{N}})_{(k_{\min})}$ as in \eqref{eq::phi-hat}. 

\end{description}
The resulting \nsns classifier  is %\textcolor{red}{here we did not assume the requirements for $m_3$, should we?}
\begin{align}\label{eq::nsn2}
\hat{\phi}_{\text{NSN}^2} (x) \,=\, \1\left\{\hat r^S_{\text{N}}(x) \geq (\hat{\Lr}^S_{\text{N}})_{(k_{\min})}\right\}.
\end{align}

\noindent\textbf{\underline{P}arametric \underline{S}creening-based \underline{N}P \underline{N}aive Bayes (\psn) classifier}

%\adc{Do we want to say it is ``essentially the same"? Maybe we can say ``The procedure of \psns is similar to that of \nsn. The major differences are as follows. In Step 1.1,  features are now selected based on $t$-statistic with $\widetilde{\cA}_{\tau}$ representing the selected features . In Step 1.2, $p_j$, $q_j$ for $j \in \widetilde{\cA}_{\tau}$ follow two class Gaussian model. "}
\noindent The \psns procedure is similar to \nsn, except the following two differences. In Step 1.1,  features are now selected based on $t$-statistics ($\widetilde{\cA}_{\tau}$ represent the index set of the selected features). In Step 1.2, $p_j$, $q_j$ for $j \in \widetilde{\cA}_{\tau}$ follow two-class Gaussian model, and the resulting parametric screening-based density ratio estimate is
$$
\hat r^S_{\text{P}}(x) = \prod_{j \in \widetilde{\cA}_{\tau}} \frac{\tilde p_{j}(x_j)}{\tilde q_{j}(x_j)}\,.
$$
The corresponding \psns classifier is thus given by%\textcolor{red}{the problem on sample size}
\begin{align}\label{eq::psn2}
\hat{\phi}_{\text{PSN}^2} (x) = \1\left\{\hat r^S_{\text{P}}(x) \geq (\hat r^S_{\text{P}})_{(k_{\min})}\right\}.
\end{align}

%\adc{Add a paragraph about PSN}
We assume the domains of all $p_j$ and $q_j$ to be $[-1,1]$ for all the following theoretical discussion. 
We will prove NP oracle inequalities for $\hat{\phi}_{\text{NSN}^2}$, and those  for $\hat{\phi}_{\text{PSN}^2}$ can be  developed similarly. Recall that by Proposition \ref{prop::R1}, we need an upper bound for $\|\hat r^S_N - r\|_{\infty}$.  Necessarily, performance of the screening step should be studied. To this end, we assume that only a small fraction of the $d$ features have marginal differentiating power.  
\begin{assumption}\label{assumption::exact-recovery}
There exists a signal set $\cA\subset \{1,\cdots,d\}$ with size  $|\cA|=s\ll d$ such that  $\inf_{j\in\mathcal{A}}\| F^{0}_j - F^{1}_j\| _{\infty} \geq D$ for some positive constant $D$, and $F^{0}_j = F^{1}_j$ for $j\notin\mathcal{A}$.

\end{assumption}

The following proposition shows that Step 1.1  achieves exact recovery ($\widehat{\mathcal{A}}_{\tau}  = \cA$) with high probability for some properly chosen $\tau$.

%\textcolor{red}{Does assumption 2 related to the following proposition? the relevant dimensionality s did not show up}
\begin{prop}[exact recovery]
\label{prop::exact recovery epsilon}
Let $\delta_1\in(0,1)$. In addition to Assumptions \ref{assumption::independence-split} and \ref{assumption::exact-recovery}, suppose $n_1 \wedge m_1 \geq 8D^{-2}\log(4d/\delta_1)$.
Then for  any $\tau \in \left[ \Delta_0, D -\Delta_0 \right],$ 
where $\Delta_0 =  \sqrt{\frac{\log(4d/\delta_1)}{2n_1}} + \sqrt{\frac{\log(4d/\delta_1)}{2m_1}}$, 
the screening substep $\mbox{Step 1.1}$ \eqref{eq::Atau} satisfies
$$\p(\widehat{\mathcal{A}}_{\tau} =\mathcal{A})\geq 1-\delta_1\,.$$
\end{prop}
Now we are ready to control the uniform deviation of density ratio estimate given in Step 1.2. 
%%%%%%%%%%%%%%%%%%%%%%%%%%%%%%%%%%%%%%
% Lemma for proving the exact recovery result
%%%%%%%%%%%%%%%%%%%%%%%%%%%%%%%%%%%%%%
%%%%%%%%%%%%%%%

%%%%%%%%%%%%%%%%%%%%%%%%%%%%%%%%%%%%%%%%%%%%%%%%%%%%%%%%%%
%	Proof-of-exact-recovery		Proof-of-exact-recovery		Proof-of-exact-recovery		Proof-of-exact-recovery		Proof-of-exact-recovery	
%%%%%%%%%%%%%%%%%%%%%%%%%%%%%%%%%%%%%%%%%%%%%%%%%%%%%%%%%%%%%%%%%%%%%%%%%%
\begin{assumption}\label{assumption::beta-valid}
The marginal densities $p_j$, $q_j\in \mathcal{P}_{\Sigma}(\beta, L, [-1,1])$ for all $j = 1, \cdots, d$, and there exists $\uderbar \mu > 0$ such that $p_j, q_j \geq \uderbar \mu$ for all $j \in \mathcal{A}$. There exists some constant $\bar C>0$, such that $\|r\|_{\infty}\leq \bar C$, and there is a uniform absolute upper bound for $\|p_j^{(l)}\|_{\infty}$ and $\|q_j^{(l)}\|_{\infty}$ for $j \in \mathcal{A}$ and $l\in [0, \lfloor \beta\rfloor]$.  Moreover, the kernel $K$ in the nonparametric density estimates  is $\beta$-valid and $L'$-Lipschitz. 
\end{assumption}
Smoothness conditions (Assumption \ref{assumption::beta-valid}) and the margin assumption were used together in the classical classification literature. However, it is not entirely obvious why Assumption \ref{assumption::beta-valid} does not render the detection condition redundant. We refer interested readers to Appendix \ref{sec::assumption3 and detection condition} for more detailed discussion.  

Let $C^1_j$ and $C^0_j$ be the constants $C$ in Lemma \ref{lemma::A1_1-dim} when applied to $p_j$ and $q_j$ respectively.  Assumption \ref{assumption::beta-valid} ensures the existence of absolute constants $C^1 \geq \sup_{j\in\mathcal{A}}C^1_j$ and $C^0 \geq \sup_{j\in\mathcal{A}}C^0_j$.

\begin{prop}[uniform deviation of density ratio estimate]\label{prop::joint_r}
Under Assumptions \ref{assumption::independence-split} - \ref{assumption::beta-valid},  
for any  $\delta_1, \delta_2 \in (0,1)$, if
$n_1 \wedge m_1 \geq 8D^{-2}\log(4d/\delta_1)$, 
$\sqrt{\frac{\log(2n_2 s/\delta_2)}{n_2{h_1}}} \leq \min(1, \underline{\mu}/C^1)$, 
$\sqrt{\frac{\log(2m_2 s/\delta_2)}{m_2{h_0}}} \leq \min(1, \underline{\mu}/C^0)$, 
and the screening threshold $\tau$ is specified as in Proposition \ref{prop::exact recovery epsilon},
we have
\begin{equation}
\label{eq::T}
 \p\left( \| \hat r^S_{\text{N}} - r \|_{\infty}
\leq  T  \right) \,\geq\, 1-\delta_1 - \delta_2\,,
\end{equation}
where $T = {B} e^B \|r\|_{\infty}$ with
$$
B
\,=\, s
\left\{ \frac{C^1\sqrt{\frac{\log(2n_2 s/\delta_2)}{n_2{h_1}}}}{ \underline{\mu} - C^1\sqrt{\frac{\log(2n_2 s/\delta_2)}{n_2{h_1}}}}
+ \frac{C^0\sqrt{\frac{\log(2m_2 s/\delta_2)}{m_2{h_0}}}}{ \underline{\mu} - C^0\sqrt{\frac{\log(2m_2 s/\delta_2)}{m_2{h_0}}}} \right\}\,.
$$
%$C_y = \sqrt{48c_1^{(y)}} + 32c_2^{(y)}+2Lc_3+L'+L+\tilde C^{(y)}\sum_{1\leq|l|\leq\lfloor\beta\rfloor}\frac{1}{l!}$ for
%$y = 0,1$, as in Lemma \ref{lemma::A1_1-dim} by taking the density function equals $p_j$ and $q_j$\textcolor{red}{This is not enough; need to consider all j's; trick to write!}, respectively.
\noindent Moreover, assume that $n_2\wedge m_2 \geq 1/\delta_2$, $|\mathcal{A}|=s \leq (n_2 \wedge m_2)^{\frac{\beta}{2(\beta+1)}}$,  and the bandwidths $h_1= (\log n_2/n_2)^{\frac{1}{2\beta+1}}$ and $h_0=(\log m_2/m_2)^{\frac{1}{2\beta+1}}$, then there exists an absolute constant $C_2>0$ such that 
$$
 \p\left[\| \hat r^S_{\text{N}} - r \|_{\infty}\leq C_2 \, s\left\{\left(\frac{\log n_2}{n_2}\right)^{\frac{\beta}{2\beta+1}} + \left(\frac{\log m_2}{m_2}\right)^{\frac{\beta}{2\beta+1}}\right\} \right] \,\geq\, 1-\delta_1-\delta_2\,.$$ 
%where $s\left[\left(\frac{\log n_2}{n_2}\right)^{\frac{\beta}{2\beta+1}} + \left(\frac{\log m_2}{m_2}\right)^{\frac{\beta}{2\beta+1}}\right]$ 
%%%%%%%%%%%%%%%%%%%%%%%%%%%%%%%%%%%%%%%%%%
\end{prop}

%%%%%%%%%%%%%%%%%%%%%%%%%%%%%%%%%%%%%%%%%%%%%%%%%%%%%%%%%%%
%	PROOF-OF-prop::joint_r	PROOF-OF-prop::joint_r	PROOF-OF-prop::joint_r	
%%%%%%%%%%%%%%%%%%%%%%%%%%%%%%%%%%%%%%%%%%%%%%%%%%%%%%%%%%%
The condition $|\mathcal{A}|=s \leq (n_2 \wedge m_2)^{\frac{\beta}{2(\beta+1)}}$ in the above proposition ensures that the upper bound of the uniform deviation diminishes as sample sizes $n_2$, $m_2$ go to infinity. Now we are in a position to present the theorem finale of  \nsn.

\begin{TH1}[NP Oracle Inequalities for $\hat \phi_{\text{NSN}^2}$]\label{theorem::2}
In addition to  Assumptions \ref{assumption::independence-split} - \ref{assumption::beta-valid},  
assume the density ratio $\Lr$ satisfies the margin assumption of order $\bar\gamma$ at level $C_{\alpha}$ and detection condition of order $\uderbar \gamma$ at level $(C_{\alpha}, \delta^*)$, both with respect to $P_0$.
For any given $\delta_1, \delta_2, \delta_3, \delta_4 \in (0,1)$, 
let the \nsn\text{ }classifier $\hat \phi_{\text{NSN}^2}$ be defined as in \eqref{eq::nsn2},
with the screening threshold $\tau$ specified as in Proposition \ref{prop::exact recovery epsilon} and kernel bandwidths $h_1= (\log n_2/n_2)^{\frac{1}{2\beta+1}}$ and $h_0=(\log m_2/m_2)^{\frac{1}{2\beta+1}}$, and $\hat \Lr^S_N$ be such that $F_{0,\hat \Lr^S_N}$ is continuous almost surely. 
For subsample sizes that satisfy 
$n_1 \wedge m_1 \geq 8D^{-2}\log(4d/\delta_1)$, 
$n_2 \wedge m_2 \geq \max\{\delta_2^{-1}, s^{\frac{2(\beta+1)}{\beta}}\}$, $\sqrt{\frac{\log(2n_2 s/\delta_2)}{n_2{h_1}}} \leq \min(1, \underline{\mu}/C^1)$, 
%$m_2 \geq \delta_2^{-1}$,
 $\sqrt{\frac{\log(2m_2 s/\delta_2)}{m_2{h_0}}} \leq \min(1, \underline{\mu}/C^0)$, \\ and
$m_3 \geq \max\{ 4/{(\alpha\delta_3)}, \delta_3^{-2}, \delta_4^{-2},(\frac{2}{5}M_1{\delta^*}^{\uderbar\gamma})^{-4}\}$, 
%and $ (n_2 \wedge m_2)^{} \geq s =|\mathcal{A}|$,
there exists an absolute constant $\tilde C>0$ such that with probability at least $1-\delta_1 - \delta_2-\delta_3-\delta_4$,
\begin{eqnarray*}
\label{eq::thm_finale_2}
&\text{(I\phantom{I})}&R_0(\hat\phi_{\text{NSN}^2}) \,\leq\, \alpha\,,\\
&\text{(II)}&R_1(\hat\phi_{\text{NSN}^2}) - R_1(\phi^*) \,\leq\, \tilde C \left\{m_3^{-(\frac{1}{4} \wedge  \frac{1+\bar\gamma}{4\uderbar{\gamma}})} + s^{1+\bar\gamma}\left(\frac{\log n_2}{n_2}\right)^{\frac{\beta(1+\bar\gamma)}{2\beta+1}}+ s^{1+\bar\gamma}\left(\frac{\log m_2}{m_2}\right)^{\frac{\beta(1+\bar\gamma)}{2\beta+1}}\right\}\,.
\end{eqnarray*}
%\begin{eqnarray*}
%\label{eq::thm_finale_2}
%&\text{(I\phantom{I})}&R_0(\hat\phi_{\text{NSN}^2}) \leq \alpha\,,\\
%&\text{(II)}&R_1(\hat\phi_{\text{NSN}^2}) - R_1(\phi^*) \,\leq\, \tilde C \left[m_3^{-(\frac{1}{4} \wedge  \frac{1+\bar\gamma}{4\uderbar{\gamma}})} + s^{1+\bar\gamma}\left\{\left(\frac{\log n_2}{n_2}\right)^{\frac{\beta(1+\bar\gamma)}{2\beta+1}}+ \left(\frac{\log m_2}{m_2}\right)^{\frac{\beta(1+\bar\gamma)}{2\beta+1}}\right\}\right]\,.
%\end{eqnarray*}
%%\begin{equation*}
%\label{eq::thm_finale_2}
%\begin{array}{ll}
%\text{(I)}&R_0(\hat\phi_{\text{NSN}^2}) \leq \alpha\,,\\
%\text{(II)}&R_1(\hat\phi_{\text{NSN}^2}) - R_1(\phi^*) \,\leq\, \tilde C \left\{m_3^{-(\frac{1}{4} \wedge  \frac{1+\bar\gamma}{4\uderbar{\gamma}})} + s^{1+\bar\gamma}\left(\frac{\log n_2}{n_2}\right)^{\frac{\beta(1+\bar\gamma)}{2\beta+1}}+ s^{1+\bar\gamma}\left(\frac{\log m_2}{m_2}\right)^{\frac{\beta(1+\bar\gamma)}{2\beta+1}}\right\}\,.
%\end{array}
%\end{equation*}
\end{TH1}

%%%%%%%%%%%%%%%%%%%%%%%%%%%%%%%%%%%%%%%%%%
%%%%%%%%%%%%%%%%%%%%%%%%%%%%%%%%%%%%%%%%%%
Theorem \ref{theorem::2} establishes the NP oracle inequalities for $\hat \phi_{\text{NSN}^2}$.  To help understand the conditions of this theorem, recall that Assumption \ref{assumption::independence-split} is about sample splitting, Assumption \ref{assumption::exact-recovery}  is on minimal signal strength for active feature set, Assumption \ref{assumption::beta-valid}  is on marginal densities and kernels in nonparametric estimates, and the margin assumption and detection condition describe the neighbourhood of the oracle decision boundary. Note that the subsample sizes $n_1$ and $m_1$ do not enter the upper bound for the excess type II error explicitly.  Instead, we have size requirements on them so that the important features are kept with high probability $1-\delta_1$ in the screening substep.  The tolerance parameter $\delta_2$ arises from the nonparametric estimation of densities, the parameter $\delta_3$ is for the tolerance on violation of type I error bound,  
 and $\delta_4$ arises from controlling $|R_0(\hat \phi_{\text{NSN}^2}) - R_0(\phi^*)|$. 

%%%%%%%%%%%%%%%%%%%%%%%%%%%%%%%%%%%%%%%%%%%%%%%%%%%%%%%%%%
%%%%%%%%%%%%%%%%%%%%%%%%%%%%%%%%%%%%%%%%%%%%%%%%%%%%%%%%%%%

\section{Numerical investigation}\label{sec::numeric}

In this section, we analyze  two simulated examples and two real datasets to demonstrate the performance of our newly proposed \nsns and \psns classifiers, in comparison with their corresponding non-screening counterparts (denoted as NN$^2$ and PN$^2$ respectively) as well as three popular methods under the classical framework: Gaussian Naive Bayes (nb), penalized logistic regression (pen-log), and Support Vector Machine (svm). We use R package ``e1071" for nb and svm, and the R package ``glmnet" for pen-log. 
To facilitate the presentation, we summarize the four Neyman-Pearson Naive Bayes classifiers in Table \ref{tb::4variants}.

%\adc{Do we want to unify the statements about "non-screening". sometimes is "non-screening", sometimes is "no screening" --- {\color{red} Unified, all ``no screening" is now ``non-screening"}}
\begin{table}[h]
\centering
\caption{A summary of the four Neyman-Pearson Naive Bayes classifiers. \label{tb::4variants}}
{\renewcommand{\arraystretch}{1.5}
\begin{tabular}{l|c|c}
& Screening-based & Non-screening\\\hline
Non-parametric
& $\hat{\phi}_{\text{NSN}^2} (x) = \1\left\{\hat r^S_{\text{N}}(x) \geq (\hat{\Lr}^S_{\text{N}})_{(k_{\min})}\right\}$ 
& $\hat{\phi}_{\text{NN}^2} (x) = \1\left\{\hat r_{\text{N}}(x) \geq (\hat{\Lr}_{\text{N}})_{(k_{\min})}\right\}$\\\hline
Parametric
& $\hat{\phi}_{\text{PSN}^2} (x) = \1\left\{\hat r^S_{\text{P}}(x) \geq (\hat{\Lr}^S_{\text{P}})_{(k_{\min})}\right\}$ 
& $\hat{\phi}_{\text{PN}^2} (x) = \1\left\{\hat r_{\text{P}}(x) \geq (\hat{\Lr}_{\text{P}})_{(k_{\min})}\right\}$\\
\end{tabular}
}
\end{table}

To train the classifiers in Table \ref{tb::4variants}, we set $\alpha = 0.05$, $\delta_1=0.05$, and $\delta_3=0.05$ throughout this section unless specified otherwise. 
In Assumption \ref{assumption::independence-split}, motivated by Proposition \ref{prop::exact recovery epsilon}, we take 
$m_1 = \min\{10\log(4d/\delta_1), m/4\}\1(\text{screening})$, $n_1=\min\{10\log(4d/\delta_1), n/2\} \1(\text{screening})$,
$m_2 = \lfloor m/2 \rfloor - m_1$, $n_2 = n - n_1$, and
$m_3 = m - \lfloor m/2\rfloor$. 

Due to the absence of information with respect to the true $p$ and $q$, the theoretical screening cutoff that achieves exact recovery is not feasible in practice. 
We resort to an empirical permutation-based approach \citep{Fan.Feng.ea.2011a} as a substitute.  Specifically, the screening substep in  \nsns is executed as follows: 
\begin{enumerate}
	\item Combine  $\mathcal{S}^0_1$ and $\mathcal{S}^1_1$ into $\{( X_{i},Y_i)\}_{i=1}^{m_1+n_1},$
 where $ X_{i} \in \mathcal{S}^0_1 \cup \mathcal{S}^1_1,$ and $Y_i$ is $ X_{i}$'s class label.
	\item Calculate the marginal $D$-statistic for each feature:
	$$
D_j = \|\hat{F}^{0}_{j}-\hat{F}^{1}_{j}\|_{\infty}, 
\quad j = 1,2,\cdots,d\,,
     $$
     where $\hat{F}^{0}_{j}(x)= \sum_{i:Y_i = 0} \1(X_{i,j}\leq x_j)$ and
$\hat{F}^{1}_{j}(x) = \sum_{i:Y_i = 1} \1(X_{i,j}\leq x_j)$.
	\item Let $\pi=\{\pi(1),\cdots,\pi(m_1+n_1)\}$ be a random permutation of $\{1,\cdots,(m_1+n_1)\}$. 
For $j=1,\cdots,d$, compute
$
D_j^{\text{null}}= \|\hat{F}^{0,\text{null}}_{j}-\hat{F}^{1,\text{null}}_{j}\|_{\infty}$,
where $\hat{F}^{0,\text{null}}_{j}(x_j)= \sum_{i:Y_{\pi(i)} = 0} \1(X_{i,j}\leq x_j)$, 
$\hat{F}^{1,\text{null}}_{j}(x_j) = \sum_{i:Y_{\pi(i)} = 1} \1(X_{i,j}\leq x_j)$.
	\item For some pre-specified $Q \in [0,1],$ let $\omega(Q)$ be the $Q$-th quantile of $\{D_j^{\text{null}}: j = 1,\cdots,d\}$ and select $\widehat{\cA} = \{j: D_j \geq \omega(Q)\}$. Here, $Q$ is a tuning parameter that keeps the percentage of noise features that pass the screening around $1-Q$.
\end{enumerate}
The same permutation idea is applied to the screening substep of \psn.
$Q$ is set at 0.95 throughout this section.

%%%%%%%%%%%%%
\subsection{Simulation}
%%%%%%%%%%%%%

Samples in both simulated examples are generated from the model
\begin{equation*}
\label{eq::independence}
p(x) = \prod_{j=1}^d p_{j}(x_j),\quad q(x) = \prod_{j=1}^d q_{j}(x_j)
\end{equation*}
at 3 different dimensions: $d \in \{10, 100,1000 \}$. 
Sparsity for $d=100$ and $1000$ is imposed by setting $p_{j} = q_{j}$ for all $j > 10$. 
Seven different training sample sizes: $m=n \in \{$200, 400, 800, 1600, 3200, 6400, 12800$\}$ are considered. The number of replications for each scenario is 1000. Test errors are estimated using the average of 1000 independent observations from each class for each replication.
%Oracle Inequality 1) specific tolerance level $\delta_3$ for selecting $k_{\min}$ is set at $\delta_3=\delta=0.0\dot{3}$. ({\color{red} I give it a new notation $\delta_3$ here in fear that otherwise might make it look as if the oracle inequality 1) specific tolerance level can only be set at $\delta$.})

%%%%%%%%%%%%%%%%%%%%%%%%%%%%%%%%%%%%%%%%%%%%%%%%%%%%%%%%%%
%%%%%%%%%%%%%%%%%%%%%%%%%%%%%%%%%%%%%%%%%%%%%%%%%%%%%%%%%%%
\subsubsection{Example 1: normals with different means}
%%%%%%%%%%%%%%%%%%%%%%%%%%%%%%%%%%%%%%%%%%%%%%%%%%%%%%%%%%%
%%%%%%%%%%%%%%%%%%%%%%%%%%%%%%%%%%%%%%%%%%%%%%%%%%%%%%%%%%%
Assume the two-class conditional densities $p \sim \n(0.5(1^\prime_{10}, 0^\prime_{d-10})^\prime, {I}_{d})$
 and $q \sim \n( 0_{d}, {I}_{d})$ where ${I}_{d}$ is the identity matrix.
%%%%%%%%%%%%%%%%%%%%%%%%%
\iffalse
The true log density ratio function $\log \Lr(x) = 0.5 \sum_{j=1}^{10} x_j - 1.25$ satisfies
$$\log \Lr(X) |Y=0 \sim \n( -1.25, 2.5), \quad \log \Lr(X) |Y= 1 \sim \n( 1.25, 2.5).$$
\fi
%%%%%%%%%%%%%%%%%%%%%%%%%%%%
At significance level $\alpha = 0.05,$ the oracle type I/II risks are
$R_0(\phi^*_{\alpha}) = 0.05$ and $R_1(\phi^*_{\alpha}) = 0.53$ respectively. 

%%%%%%%%%%%%%%%%%%%%%%%%%%%%%%%%%%%%%%%%%%%%
%
%  Screening Ex1
%
We first evaluate the screening performance of \psns and \nsns with results presented in Table \ref{tb::screening1}. 
Both $t$-statistic (in \psn) and $D$-statistic (in \nsn) are able to  pick up most of the true signals while keeping the false positive rates at around $1-Q$.\\

\begin{table}[h]
\caption{Average screening performance summarized over 1000 independent replications at sample sizes $m=n=400$ and $Q = 0.95$ with standard errors in parentheses. \label{tb::screening1}}
\centering
\scriptsize{
\begin{tabular}{lrr|rr|rr}
&\multicolumn{2}{c|}{\# of selected features} & \multicolumn{2}{c|}{\# of missed signals}& \multicolumn{2}{c}{\# of false positive} \\
\hline
$d$ 
& \multicolumn{1}{c}{$t$-stat } & \multicolumn{1}{c|}{$D$-stat} 
& \multicolumn{1}{c}{$t$-stat } & \multicolumn{1}{c|}{$D$-stat} 
& \multicolumn{1}{c}{$t$-stat } & \multicolumn{1}{c}{$D$-stat}
\\\hline
10	
&9.11 (1.14)& 8.11 (1.63)
&0.89 (1.14)&1.89 (1.63)
&0 (0)&0 (0)	\\
100
&14.64 (3.46)&12.43 (3.38)
&0.78 (0.90)&2.00 (1.39)	
&5.43 (3.17)&4.43 (2.77)\\
1000
&59.99 (9.77)&58.82 (9.87)
&0.48 (0.66)&1.14 (1.05)
&50.47 (9.71)&49.96 (9.78)	\\\end{tabular}}
\end{table}

\vspace{-0.2in}
\begin{figure}[!h]
  \caption{Average errors of $\hat{\phi}$'s over 1000 independent replications for each combination of $(d,m,n)$.
 \label{fig::Ex1_risks}}
  \centering
\includegraphics[width=\textwidth]{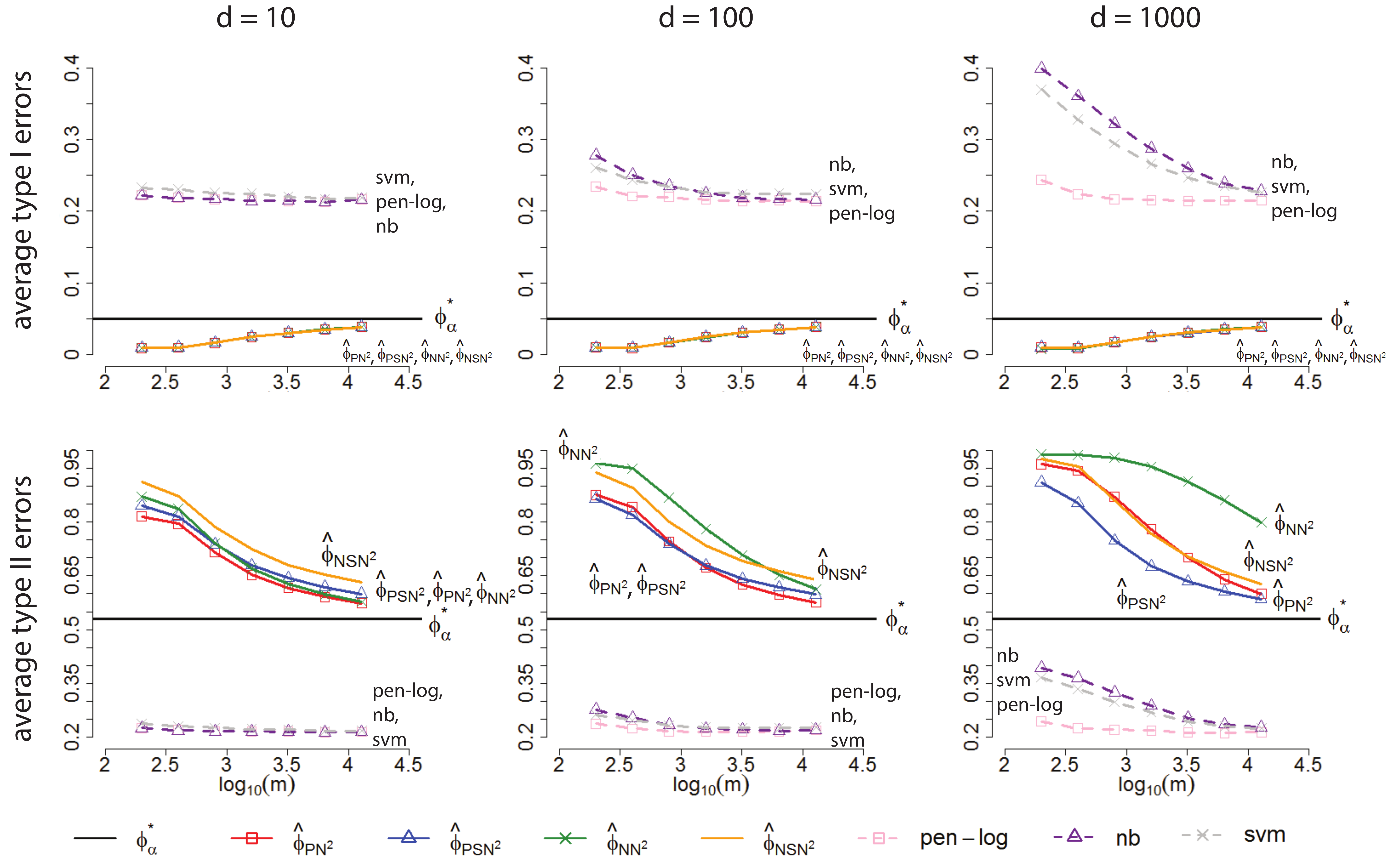}
\end{figure}

%%%%%%%%%%%%%%%%%%%%%%%%%%
%
% Ex1: Average Error Rates 
%
%%%%%%%%%%%%%%%%%%%%%%%%%%%

We then move on to evaluate the trend of type I and type II errors as the sample size increases in Figure \ref{fig::Ex1_risks}.  
All the Neyman-Pearson based classifiers have type I error approaching $\alpha$ from below as sample size increases and they have similar type I errors at each sample size. However, nb, pen-log and svm all lead to a type I error larger than $\alpha$.

By enlarging the second row of Figure \ref{fig::Ex1_risks}, one would observe the differences in type II errors among \pn, \psn, \nn,  \nsn. In the case of $d=10$ when all features are signals, \pns performs the best throughout all sample sizes since it assumes the correct model without the unnecessary screening substep. When sample size is small, \psns outperforms \nn, but \nns gradually catches up on larger samples. In the case of $d = 100$, screening helps \psns to take the lead at low sample sizes. The advantage of screening fades off as the sample size increases. In the case of $d=1000$, \psns dominates all other three classifiers throughout the sample size range we investigate.  
%In the case of $d = 1000$, screening allows $\phins$and $\phips$to take the lead over their non-screening counterparts all the way through $m=n=128,00$, where $\phip$finally manages to catch up with $\phips$as a result of the power gained from correct model assumptions. The gain from screening also makes up for the lack of model in the case of $\phins$, and makes it at least as good as $\phip$, if not better, most of the time until $m,n$ become larger than $10^{3.5} (\approx 3200)$. 

Overall, the advantage of \psns over \nsn, and \pns over \nns are uniform across all dimensions and sample sizes.
This is consistent with the intuition that when the data are from a two-class Gaussian model, the parametric methods lead to more efficient estimators than nonparametric counterparts.
%%%%%%%%%%%%%%%%%%%%%%%%%%%

%The trend that $\hat R_0$ and $\hat R_1$ gradually approach $R_0(\phi^*)$ and $R_1(\phi^*)$ as the sample size increases is consistent with the theoretical results in Proposition \ref{prop::R0} and @@@@.
%%%%%%%%%%%%%%%%%%%%%%%%

%%%%%%%%%%%%%%%%%%%%%%%%%%%%%%%%%%%%%%%%%%%%%%%%%%%%%%%%%%%%
%%%%%%%%%%%%%%%%%%%%%%%%%%%%%%%%%%%%%%%%%%%%%%%%%%%%%%%%%%%
\subsubsection{Example 2: normal vs. mixture normal}
%%%%%%%%%%%%%%%%%%%%%%%%%%%%%%%%%%%%%%%%%%%%%%%%%%%%%%%%%%%
%%%%%%%%%%%%%%%%%%%%%%%%%%%%%%%%%%%%%%%%%%%%%%%%%%%%%%%%%%%%
Normality assumption is violated in the second example. Assume $p \sim 0.5\n( a,\Sigma) + 0.5 \n (- a, \Sigma)$ and $q \sim \n( 0_d,I_d)$,
where $ a = (\frac{3}{\sqrt{10}} 1^\prime_{10},  0^\prime_{d-10})^\prime$, 
$\Sigma = \left( \begin{array}{cc}10^{-1} I_{10} & 0\\ 0 & I_{d-10} \end{array}\right)$. At significance level $\alpha = 0.05,$ 
the oracle type I/II risks are
$R_0(\phi^*_{\alpha}) = 0.05$ and $R_1(\phi^*_{\alpha}) = 0.027$ respectively.

%%%%%%%%%%%
%
% Screening Ex2 
%
%%%%%%%%%%%
The performance of the screening substep of \psns and \nsns is shown in Table \ref{tb::scr2}. 
While both screening methods keep the false positive rates at around $1-Q$, the parametric screening method (\psn) with $t$-statistic misses almost all signals. This is not surprising since $t$-statistics rank features by differences in means and the two groups have exactly the same marginal mean and variance across all dimensions.

\begin{table}[h]
\caption{
Average screening performance summarized over 1000 independent replications at sample sizes $m=n=400$ and $Q = 0.95$ with standard errors in parentheses. \label{tb::scr2} 
}
\centering
\scriptsize{
\begin{tabular}{lrr|rr|rr}
&\multicolumn{2}{c|}{\# of selected features} & \multicolumn{2}{c|}{\# of missed signals}& \multicolumn{2}{c}{\# of false positive} \\
\hline
$d$ 
& \multicolumn{1}{c}{$t$-stat } & \multicolumn{1}{c|}{$D$-stat} 
& \multicolumn{1}{c}{$t$-stat } & \multicolumn{1}{c|}{$D$-stat} 
& \multicolumn{1}{c}{$t$-stat } & \multicolumn{1}{c}{$D$-stat}
\\\hline
10	 
&1.76 (1.53)&8.13 (1.83)	
& 8.24 (1.53)&1.87 (1.83)
&	0 (0) &0 (0)	\\
100  
&5.93 (3.44)&11.96 (3.57)	
& 9.38 (0.80)&2.34 (1.59) 
& 5.31 (3.17)& 4.29 (2.68)	\\
1000
&50.69 (9.60)&58.78 (9.87)
& 9.50 (0.69)&1.26 (1.04)
&50.19 (9.51)&50.04 (9.62)	\\

\end{tabular}}
\end{table}

%%%%%%%%%%%%%%%%%%%%%%%%
%
% Average risks Ex2
%
%%%%%%%%%%%%%%%%%%%%%%%%
\begin{figure}[!h]
  \caption{Average error rates of $\hat{\phi}$'s over 1000 independent replications for each combination of $(d,m,n)$.
Error rates are computed as the average of 1000 independent testing data points from each class in each replication, and then average over replications. \label{fig::Ex2_risks}}
  \centering
\includegraphics[width=\textwidth]{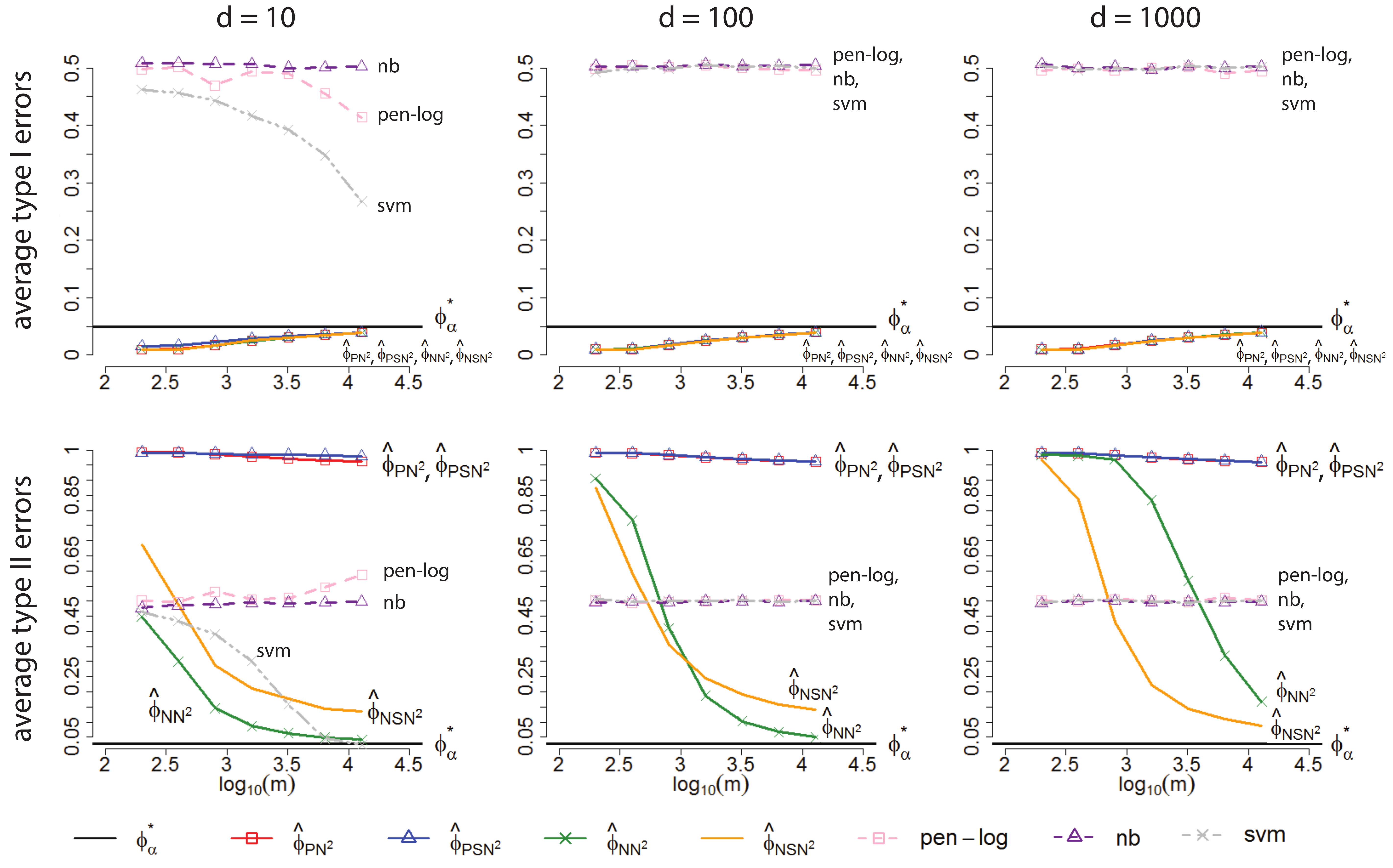}
\end{figure}

%\adc{Change the wording for the next paragraph}
Figure  \ref{fig::Ex2_risks} presents the average error rates. 
The same reason that causes the above fiasco of $t$-statistic screening
reduces \psns and \pns to nothing more than, if not less than, two unfair random coins with probability 0.05 of landing 1, while the behaviors of nb and pen-log bear more resemblance to that of fair random coins.
This fundamental difference is due to that the classical framework aims to minimize the overall risk, and therefore tends to distribute errors evenly when the sample size for the two classes are about the same. 
The  \nsns and \nns based on nonparametric assumptions, on the other hand, perform very well on non-normal data. 
Their difference in type II error performances  are similar as in Example 1.

%\subsection{Recommendations}
From the two simulation examples, it is clear that the
screening-based \nsns and \psns \\ exhibit advantages over their non-screening counterparts under 
high-dimensional settings. When the normality assumption is violated, and the sample sizes are reasonably large for efficient kernel estimates, 
\nsns prevails over \psns. 
As a rule of thumb, for high-dimensional classification problems that emphasize type I error control, we recommend \nsns if the sample size is relatively large and \psns otherwise. 

%%%%%%%%%%%%%%%%%%%%%%%%%%%%%%%%%%%%%%%%%%%%%%%%%%%%%%%%%%
%%%%%%%%%%%%%%%%%%%%%%%%%%%%%%%%%%%%%%%%%%%%%%%%%%%%%%%%%%
\subsection{Real data analysis}
%%%%%%%%%%%%%%%%%%%%%%%%%%%%%%%%%%%%%%%%%%%%%%%%%%%%%%%%%%
%%%%%%%%%%%%%%%%%%%%%%%%%%%%%%%%%%%%%%%%%%%%%%%%%%%%%%%%%%
In addition to the neuroblastoma dataset analyzed in the introduction, we now demonstrate the performance of \psns and \nsns for targeted asymmetric error control on two additional real datasets.

%%%%%%%%%%%%%%%
\subsubsection{p53 mutants dataset}
%%%%%%%%%%%%%%%
The p53 mutants dataset \citep{danziger2006functional} contains $d=5407$ attributes extracted from biophysical experiments for 16772 mutant p53 proteins, among which 143 are determined as ``active" and the rest as ``inactive" via in vivo assays.

All 143 active samples and the first $1500$ inactive samples are included in our analysis.
We treat the active class as class 0 and aimed to control the error of missing an active under $\alpha = 0.05$.
 %Due to the relatively small size, 
%we set $\delta_3 = 0.1$. 
This dataset is split  into a training set with 100 observations from the active class and 1000 observations from the inactive class, and a testing set with the remaining observations. \psns is used as the representative of our proposed methods, as the class 0 sample size is small for nonparametric methods.
The average type I and type II errors over 1000 random splits are shown in Table \ref{tb::realData}. Compared with pen-log, nb and svm,
\psns performs much better in controlling the type I error. 

\begin{table}[h]
\caption{Average errors over 1000 random splits  with standard errors in parentheses. $\alpha = 0.05$, $\delta_1 = 0.05$, $Q=0.95$, and $\delta_3 = 0.1$.
 \label{tb::realData}}
\vskip 6pt
\centering
{\renewcommand{\arraystretch}{1.5}
\begin{tabular}{lrrrr}
&\multicolumn{1}{c}{\psn}
&\multicolumn{1}{c}{pen-log}
&\multicolumn{1}{c}{nb}
&\multicolumn{1}{c}{svm}\\\hline
type I  &\underline{.019 (.028)} & .162 (.060)  &.056 (.034) &.484 (.222)\\
type II &.461 (.291)&.010 (.004)&.458 (.033) &.004 (.003)
\end{tabular}
}
\end{table}

%%%%%%%%%%%%%%%%%%%%%%%%%
%
\subsubsection{Email spam dataset}
Now, we consider an e-mail spam dataset available at \url{https://archive.ics.uci.edu/ml/datasets/Spambase}, which contains 4601 observations with 57 features, among which 2788 are class 0 (non-spam) and 1813 are class 1 (spam).
We first standardize each feature and add 5000 synthetic features consisting of independent $\n(0,1)$ variables to make the problem more challenging.
The augmented data has $n=4601$ observations with $d=5057$ features. 
This augmented dataset is split  into a training set with 1000 observations from each class and a testing set with the remaining observations. We use \nsns since the sample size is relatively large.
The average type I and type II errors over 1000 random splits are shown in Table \ref{tb::spam}.  

To evaluate the flexibility of \nsns in terms of prioritized error control, we also  report the performance when the priority is switched to control the type II error below $\alpha=0.05$. The results in Table \ref{tb::spam} demonstrate that  \nsns is able to control either type I or type II error depending on the specific need of the practitioner.
%
%The rates of {\nsn-$R_0$} violating $\{R_0 \leq \alpha\}$ 
%and {\nsn-$R_1$} violating $\{R_1 \leq \alpha\}$ are 0 and $0.016$, respectively; both satisfy Oracle Inequality 1) as we wish.

\begin{table}[h]
\caption{Average errors over 1000 random splits with standard errors in parentheses. 
 $\alpha = 0.05$, $\delta_1 = 0.05$, $Q=0.95$, and $\delta_3=0.05$. The suffix after \nsns indicates the type of error it targets to control under $\alpha$.
 \label{tb::spam}}
 
\vskip 4pt

\centering
{\renewcommand{\arraystretch}{1.5}
\begin{tabular}{lrrrrr}
&\multicolumn{1}{c}{\nsn-$R_0$}
&\multicolumn{1}{c}{\nsn-$R_1$}
&\multicolumn{1}{c}{pen-log}
&\multicolumn{1}{c}{nb}
&\multicolumn{1}{c}{svm}\\\hline
type I &.\underline{019 (.007)}&.488 (.078)&.064 (.007)&.444 (.018)&.203 (.013)\\
type II &.439 (.057)&\underline{.020 (.009)}&.133 (.015)&.054 (.008)&.235 (.017)
\end{tabular}
}
\end{table}

\section{Discussion\label{sec::discussion}}

The Neyman-Pearson classification framework is an important and interesting paradigm to explore beyond the Naive Bayes models considered in this work.  For example, we can relax the independence assumption on \psn, and consider a general covariance matrix. Also, we can consider NP-type classifiers with decision boundaries involving feature interactions.  It is also worthwhile to study the non-probabilistic approaches under high-dimensional NP paradigm. Methods of potential interest include the $k$ nearest neighbor \citep{weiss2010text}  and the centroid based classifiers \citep{Tibshirani.Hastie.ea.2002, hall2010optimal}. However, the NP oracle inequalities are likely to be replaced by a new theoretical formulation for these methods.

A benefit of the present approach is that, for any given estimator $\hat{r}$, we have a uniform method to determine the proper threshold level in the plug-in classifiers.  
However, it would be interesting to develop new ways to estimate the threshold level $C_{\alpha}$ that is adaptive to the  particular method used to approximate the density ratio $r$.  
%Applications of these high-dimensional NP-type classifiers are potential outcomes of collaboration with computational biologists and other scientists.   
\appendix

\section{Technical Lemmas and Proofs}\label{sec::supple}
Let $\text{Bin.cdf}_{n, \,p}(\cdot)$ denote the \textsc{cdf} of Bin$(n,p)$, and $\text{Beta.cdf}_{a,\, b}(\cdot)$ denote the \textsc{cdf} of Beta$(a,b)$.
The following lemma proves a duality between the beta and binomial distributions. 

\begin{lem}[Beta-binomial duality]
\label{lem::betaBinomial}
For any $p\in[0,1]$ and $k \in \{1, \ldots, n\}$, it holds that 
\begin{eqnarray*}
1- \text{Bin.cdf}_{n, \,p}(k-1) &=& \text{Beta.cdf}_{k, \,n+1-k}\left( p\right)\,. 
\end{eqnarray*}
\end{lem}

\begin{proof}[Proof of Lemma \ref{lem::betaBinomial}]
Let $U_1, \ldots, U_n$ be $n$ {i.i.d.\,}Uniform$[0,1]$.
For any $p \in [0,1]$,  let $N_p = \sum_{i=1}^n \1{\{U_i \leq p\}}$ denote the number of $U_i$'s that are less or equal to $p$. 
Given
$$\p( \1{\{U_i \leq p\}} =1) \,\,=\,\, \p(U_i \leq p) \,\,=\,\, p\,,\quad \1{\{U_i \leq p\}}\,\, \sim\,\, \text{Bern}(p) \quad \forall i\,,$$
we have $N_p \sim \text{Bin}(n,p)$, and therefore
\begin{eqnarray}
\label{eq::duality1}
\p\left( N_p \geq k\right)\,\,=\,\,  1-  \p\left( N_p \leq k-1\right)\,\,=\,\, 1- \text{Bin.cdf}_{n,p}(k-1)\,.
\end{eqnarray}
On the other hand, 
let $U_{(k)}$ denote the $k$-th order statistic of $\{U_i\}_{i=1}^n$. 
It follows from the definition of order statistics that 
\begin{eqnarray}
\label{eq::duality2}
\left\{ N_p \geq k\right\}
\,\,=\,\, \left\{\text{at least $k$ of $U_1, \ldots, U_n$ are less or equal to $p$}\right\}
\,\,=\,\, \left\{ U_{(k)} \leq p\right\}\,.
\end{eqnarray} 
Combining \eqref{eq::duality1} with \eqref{eq::duality2}
yields 
\begin{eqnarray*}
1- \text{Bin.cdf}_{n,p}(k-1)
\,\,=\,\, \p\left( N_p \geq k\right) 
\,\,=\,\, \p\left( U_{(k)} \leq p\right)
\,\, =\,\, \text{Beta.cdf}_{k, \,n+1-k}\left( p\right)\,,
\end{eqnarray*}
where the last equality follows from
$U_{(k)} \sim \text{Beta}(k, n+1-k)$ $(k = 1, \ldots, n)$ 
as a direct implication of R\'{e}nyi's representation.
This completes the proof. 
\end{proof}

\begin{lem}
\label{lem::calculus}
Let $Z$ be a random variable from \textsc{cdf} $F$. 
We have
\begin{eqnarray}
\label{eq::calculus}
P_F\left\{F(Z) < \delta\right\} \,\,\leq\,\, \delta\,, \quad
P_F\left\{F(Z) > \delta\right\} \,\,\geq\,\, 1-\delta\,\,\quad 
\forall \delta \in [0,1]\,.
\end{eqnarray}
For continuous $F$, the inequality becomes equality as 
\begin{eqnarray}
\label{eq::calculus2}
P_F\left\{F(Z) < \delta\right\} \,\,=\,\, \delta\,, \quad
P_F\left\{F(Z) > \delta\right\} \,\,=\,\, 1-\delta\,\,\quad  \forall \delta \in [0,1]\,.
\end{eqnarray}. 
\end{lem}

\begin{proof}
Let $t_1 = \min\left\{ t: F(t) \geq \delta \right\}$.  Given the right continuity of $F$, it can be easily proved by contradiction that  i) $F(t_1-) = F(t_1) = \delta$ if $F$ is continuous at $t_1$, and ii) $F(t_1-) < \delta \leq F(t_1)$ if $F$ is discontinuous at $t_1$.
Thus, 
$$P_F\{F(Z) < \delta\} \,=\, P_F(Z <  t_1) \,=\, F(t_1-) \,\leq\, \delta\,.$$
Likewise, let $t_2 =\inf\left\{ t: F(t) > \delta \right\}$. We have i) $F(t_2-) = F(t_2) = \delta$ if $F$ is continuous at $t_2$, and ii) $F(t_2-) < \delta \leq F(t_2)$ if $F$ is discontinuous at $t_2$. 
As a result, 
$$
P_F\{ F(Z) > \delta\} \,=\, P_F\left\{ Z \geq t_2 \right\} \,=\, 1 - P_F\left\{ Z < t_2 \right\} \,\geq\, 1- \delta\,.$$
This completes the proof. 
\end{proof}

\begin{lem}
\label{lem::new}
Let $\mathcal{S} = \{Z_i\}_{i=1}^n$ be a set $n$ {i.i.d.\,}random variables from distribution $F$, and let $Z_{(k)}$ denote its $k$-th order statistic $(k = 1, \ldots, n)$. 
For any $\delta \in (0,1)$, 
the probability of a new, independent realization $Z$ from $F$ to be greater than $Z_{(k)}$ satisfies
\begin{eqnarray}
&&\p\left\{ P_F\left( Z > Z_{(k)} \mid \mathcal{S} \right) > \delta \right\}
\,\leq\, 
1- \text{Bin.cdf}_{n, 1-\delta}(k-1)\,,\label{eq::wangzheng}\\
%%%%%%%%%%%%%%%%%%%%%%
&&\p\left\{ P_F\left( Z > Z_{(k)} \mid \mathcal{S} \right) < \delta \right\}
\,\geq\, 
1- \text{Bin.cdf}_{n,\delta}(n-k) \,=\,\text{Bin.cdf}_{n,1-\delta}(k-1)\label{eq::wangzheng2}\,.
\end{eqnarray}
The inequalities become equalities if $F$ is continuous.
\end{lem}

\begin{proof}[Proof of Lemma  \ref{lem::new}.]
Rewrite the left-hand side of \eqref{eq::wangzheng} as
\begin{eqnarray}
\label{eq::LHS}
&&\p\left\{ P_F\left( Z > Z_{(k)} \mid \mathcal{S} \right) > \delta \right\}
\,\,=\,\,
\p\left\{ 1- 
P_F\left( Z \leq Z_{(k)}  \mid \mathcal{S} \right) > \delta \right\}\nonumber\\
&=&
\p\left\{ 1 - F( Z_{(k)}) > \delta \right\}\,\,=\,\,
\p\left\{ F( Z_{(k)}) < 1- \delta \right\}.
\end{eqnarray}
To bound the probability of $\{F( Z_{(k)}) < 1- \delta\}$, 
let $N_{1-\delta} \,\,=\,\, \sum_{i=1}^n \1_{\{F(Z_i) < 1-\delta\}}$ denote the number of  $F(Z_i)$'s that are less than $1-\delta$.
It follows from 
$F(Z_{(1)}) \leq F(Z_{(2)})\leq \ldots\leq F(Z_{(n)})$ that
\begin{eqnarray}
\label{eq::equiv}
&&\left\{ F( Z_{(k)}) < 1- \delta \right\} \,\,=\,\, \left\{ F( Z_{(i)}) < 1- \delta\,,\,\, i = 1,\ldots, k \right\}
\,=\, \left\{ N_{1-\delta} \geq k \right\}\nonumber
\,,\\
&&\p\left\{ F( Z_{(k)}) < 1- \delta \right\}\,\,=\,\,\p\left( N_{1-\delta} \geq k \right)\,.
\end{eqnarray}
Let $\tau = P_F\left\{F(Z_1) < 1-\delta\right\}$ denote the success probability of $N_{1-\delta}$ as a binomial. 
It follows from \eqref{eq::calculus} that $\tau \leq 1-\delta$.
Given $\text{Bin.cdf}_{n,p}(k-1)$ being decreasing in $p$ for any fixed $n$ and $k$, we have
\begin{eqnarray}
\label{eq::binomial}
\p\left( N_{1-\delta} \geq k \right)
\,\,=\,\, 1 - \text{Bin.cdf}_{n,\tau}(k-1) \,\,\leq 1- \text{Bin.cdf}_{n, \,1-\delta}(k-1)
\end{eqnarray}
as a result of 
The equalities hold for continuous $F$. 
Connecting \eqref{eq::LHS}, \eqref{eq::equiv}, and \eqref{eq::binomial} together yields 
\begin{eqnarray*}
\p\left\{ P_F\left( Z > Z_{(k)} \mid \mathcal{S} \right) > \delta \right\}
&=&
\p\left\{ F( Z_{(k)}) < 1- \delta \right\}
\,\,=\,\,
%\p\left\{ \text{at least $k$ of $\{F(Z_i)\}_{i=1}^n$ are strictly less than $(1-\delta)$}\right\}\,, \label{eq::binomial}\\&=& 
\p\left( N_{1-\delta} \geq k \right) \\
&\leq&1 - \text{Bin.cdf}_{n,\,1-\delta}(k-1)\,. 
\end{eqnarray*}
%%%%%%%%%
%%%%%%%%%
Likewise, 
let $M_{1-\delta} \,\,=\,\, \sum_{i=1}^n \1_{\{F(Z_i) > 1-\delta\}}$ be a binomial random variable with size $n$ and success rate $\tau^\prime \,=\,P_F\{F(Z_i) > 1-\delta\}\,\geq \, \delta$ that represents the number of  $F(Z_i)$'s that are greater than $1-\delta$.
The left-hand side of \eqref{eq::wangzheng2} can be rewritten as
\begin{eqnarray}
\label{eq::LHS2}
&&\p\left\{ P_F\left( Z > Z_{(k)} \mid \mathcal{S} \right)< \delta \right\}
\,=\,
\p\left\{ F( Z_{(k)}) > 1- \delta \right\}\nonumber\\
&=&
\p\left\{ F( Z_{(i)})  > 1- \delta\,,\,\, i = k, \ldots, n \right\}
\,=\, \p\left\{ M_{1-\delta} \geq n+1-k \right\}\nonumber\\
&=&1 - \p\left\{ M_{1-\delta} \leq n-k \right\} \,=\, 1-\text{Bin.cdf}_{n,\tau^\prime}(n-k)\nonumber\\
&\geq& 1-\text{Bin.cdf}_{n,\,\delta}(n-k).
\end{eqnarray}
This completes the proof. 
\end{proof}

\begin{proof}[Proof of Proposition \ref{prop::general delta}]
Letting $Z_i = \hat{r}_i$, $n = m_3$ in Lemma \eqref{lem::new} yields
\begin{eqnarray*}
\p\{R_0 ( \hat{\phi}_k)> \delta\} &\leq& 
1- \text{Bin.cdf}_{m_3,\,1-\delta}(k-1)\,.
\end{eqnarray*}
This, together with Lemma \ref{lem::betaBinomial}, completes the proof. 
\end{proof}

%\begin{lem}
%\label{lem::beta_dis}
%Given $\mathcal{S} = \{Z_i\}_{i=1}^n$ as $n$ independent realizations from some distribution with \emph{continuous} \textsc{cdf} $F$,
%let  $Z_{(k)}$ denote the $k$th order statistic of $\mathcal{S}$ $(k \in \{1,\cdots,n\})$.  
%For any $k \in \{1, \ldots, n\}$, 
%the probability of a new, independent realization $Z$ from $F$ to be greater than $Z_{(k)}$ satisfies 
%$$P_F(Z>Z_{(k)}\mid \mathcal{S})  \,\,\sim\,\, \mbox{Beta}(n+1-k,k)\,.$$
%\end{lem}
%
%\begin{proof}[Proof of Lemma \ref{lem::beta_dis}] Let $U_i = F(Z_i),$ $i = 1,\cdots,n.$ 
%Clearly,
%$\{U_i\}_{i=1}^n\stackrel{i.i.d.}{\sim} Unif[0,1]$.
%Since $F$ is non-decreasing and thus preserves orders, we have 
%$F(Z_{(k)}) = U_{(k)} \sim \text{Beta}(k, n+1-k),$
%where $U_{(k)}$ is the $k$th order statistic of $\{U_i\}_{i=1}^n$.
%Therefore, 
%$ 1 - F(Z_{(k)}) \sim \text{Beta}(n+1-k,k).$
%\end{proof}

\begin{lem}
\label{lem::Beta}
For random variable $Z \sim \text{Beta}(a,b),$ 
and any $\epsilon > 0,$ 
we have
\begin{align}
\p\{ Z > (1+\epsilon)\E Z\} < 
\p( \left| Z - \E Z \right| > \epsilon \E Z ) < \frac{b\epsilon^{-2}}{(a+b+1)a}.
\end{align}
\end{lem}

%%%%%%%%%%%%%%%%%%%%%%%%%%%%%%%%%%%%%%%%%%%%%
\begin{proof}[Proof of Lemma \ref{lem::Beta}]
By Chebyshev inequality,
\begin{eqnarray*}
\p ( \left| Z - \E Z\right| > \epsilon \E Z )
\leq \frac{\text{var}(Z)}{ (\epsilon\E Z)^2}
= \frac{ab}{(a+b)^2(a+b+1)}\left(\frac{\epsilon a}{a+b}\right)^{-2}
= \frac{b\epsilon^{-2}}{(a+b+1)a}.
\end{eqnarray*}
\end{proof}

%%%%%%%%%%%%%%%%%%%%%%%%%%%%%%%%%%%%%%
%\item as $m_3\to \infty$, \begin{eqnarray*}
%R_0( \hat{\phi}  ) 
%\overset{p}{\longrightarrow} R_0( \phiStar).
%\end{eqnarray*}
%\end{lem}

%%%%%%%%%%%%%%%%%%%%
% PROOF
%%%%%%%%%%%%%%%%%%%%

\begin{proof}[Proof of Proposition \ref{prop::general k}]
Let $B$ be a realization from Beta$(k,m_3+1-k)$. 
It follows from Proposition \ref{prop::general delta} that
\begin{eqnarray*}
\p\{ R_0(\hat{\phi}_k)> g(\delta_3, m_3, k) \} 
&\leq& \text{Beta.cdf}_{k, m_3+1 -k}\left\{1-g(\delta_3, m_3, k)\right\}\\
&=& \p\{ B \leq  1 - g(\delta_3, m_3, k) \}\,\,=\,\, \p\{ 1- B \geq g(\delta_3, m_3, k)\}
\end{eqnarray*}
for any $k \in \{1,\cdots,m_3\}$ and $\hat r$, with $1-B \sim \text{Beta}(m_3+1-k,k)$. 
Letting $a = m_3+1-k$, $b =k$, and $\epsilon =  k^{1/2}\{\delta_3(m_3+2)(m_3 + 1 -k)\}^{-1/2}$ in Lemma \ref{lem::Beta} yields
\begin{eqnarray*}
\p\{R_0(\hat{\phi}_k) > g(\delta_3, m_3, k)\} \,\leq\, \delta_3\,,
\end{eqnarray*}
where $$g(\delta_3, m_3, k) = (1+\epsilon)\left( \frac{m_3+1-k}{m_3+1}\right)  = 
\frac{ m_3+1-k}{m_3+1} +  
\sqrt{\frac{k(m_3+1-k)}{\delta_3(m_3+2)(m_3+1)^2}}\,.$$
This completes the proof.

\end{proof}

\begin{proof}[Proof of Proposition \ref{prop::kmin}]
By some basic algebra we have
\begin{eqnarray*}
&&A_{\alpha,\delta_3}(m_3) -(1-\alpha) 
=
\frac{-1+2\alpha + \sqrt{1+4\delta_3(1-\alpha)\alpha(m_3+2)}}{2\left\{\delta_3(m_3+2)+1 \right\}}
> 0\,,\\
&&A_{\alpha,\delta_3}(m_3) -1 
=
\frac{-1-2\delta_3(m_3+2)\alpha + \sqrt{1+4\delta_3(1-\alpha)\alpha(m_3+2)}}{2\left\{\delta_3(m_3+2)+1 \right\}}
< 0\,,
\end{eqnarray*}
and
\begin{eqnarray*}
&&g(\delta_3, m_3, k) = \frac{m_3+1-k}{m_3+1} + \sqrt{\frac{k(m_3+1-k)}{\delta_3(m_3+2)(m_3+1)^2}} \leq \alpha\\
&\Leftrightarrow& 
\left\{ 
\begin{array}{ll}
k-(1-\alpha)(m_3+1) \geq 0,\\\\
\left\{\delta_3(m_3+2)+1 \right\} \left(\frac{k}{m_3+1}\right)^2 - 
\left\{1+2\delta_3(m_3+2)(1-\alpha)\right\}\left(\frac{k}{m_3+1}\right) \\+
\delta_3(m_3+2)(1-\alpha)^2 \geq 0
\end{array}
\right.
\\\\
&\Leftrightarrow&
k \geq (m_3+1) \max\left\{
1-\alpha, A_{\alpha,\delta_3}(m_3)
 \right\}
\\
&\Leftrightarrow&
k \geq (m_3+1)A_{\alpha,\delta_3}(m_3)\,.
\end{eqnarray*}
Thus,
\begin{align*}
k_{\min}(\alpha,\delta_3,m_3) &= \left\lceil (m_3+1)A_{\alpha,\delta_3}(m_3)\right\rceil \\
&\in \left[(m_3+1)A_{\alpha,\delta_3}(m_3),  (m_3+1)A_{\alpha,\delta_3}(m_3) +1\right]\,.
\end{align*}
Since
$A_{\alpha,\delta_3}(m_3) \to 1-\alpha,$ as $m_3 \to\infty,$
it follows from sandwich rule that
\begin{equation*}
\lim_{m_3\to\infty}\frac{k_{\min}(\alpha,\delta_3,m_3)}{m_3}\,=\, \lim_{m_3\to\infty} A_{\alpha,\delta_3}(m_3) \,=\, 1 - \alpha\,.
\end{equation*}
We have
$k_{\min}(\alpha, \delta_3, m_3) \in \mathcal{K}(\alpha, \delta_3, m_3)$ 
$\left( \Leftrightarrow k_{\min}(\alpha, \delta_3, m_3) \leq m_3\right)$ as long as 
\begin{equation}
\label{eq::target}
(m_3+1)A_{\alpha, \delta_3}(m_3)  + 1 \,\leq\, m_3 \quad \Leftrightarrow \quad \left( 1-\alpha \,\leq\, \right)\,\,  A_{\alpha, \delta_3}(m_3) \,\leq\, \frac{m_3 - 1}{m_3+1} \,.
\end{equation}
For any $\Delta \in (0, \alpha)$,  a sufficient condition for \eqref{eq::target} is
\begin{eqnarray*}
\label{eq::simplified}
\frac{m_3 - 1}{m_3+1} \geq 1-\Delta, \quad A_{\alpha, \delta_3}(m_3)\leq 1-\Delta\,,
\end{eqnarray*}
which can be further simplified as
\begin{align*}
m_3 \geq \frac{2}{\Delta}-1, \quad m_3 \geq x^* - 2\,,
\end{align*}
where 
\begin{eqnarray*}
x^* &=& 
%\frac{- \left(2\Delta^2 + \alpha^2 - 2\alpha\Delta - \Delta \right)  + \sqrt{\left(2\Delta^2 + \alpha^2 - 2\alpha\Delta - \Delta \right)^2 + 4 (\alpha-\Delta)^2\Delta(1-\Delta) }}{2(\alpha-\Delta)^2 \delta_3}\\
%&=&
\frac{-2\Delta^2 - \alpha^2 + 2\alpha\Delta +\Delta  +(1-2\alpha)\Delta + \alpha^2 }{2(\alpha-\Delta)^2 \delta_3}=\frac{\Delta(1-\Delta)}{(\alpha-\Delta)^2 \delta_3}
\end{eqnarray*}
is the positive root of the quadratic equation
$$
 (\alpha-\Delta)^2\delta_3^2x^2 +\delta_3\left(2\Delta^2 + \alpha^2 - 2\alpha\Delta - \Delta \right) x
- \Delta(1-\Delta) = 0\,.
$$  
Thus,  a sufficient condition for \eqref{eq::target} is
\begin{equation*}
m_3 \geq \max\left\{ \frac{\Delta(1-\Delta)}{(\alpha-\Delta)^2 \delta_3} -2, \,\frac{2}{\Delta} - 1  \right\}\,.
\end{equation*}
Setting $\Delta= \alpha/2$ yields
\begin{equation*}
 \max\left\{ \frac{\Delta(1-\Delta)}{(\alpha-\Delta)^2}\cdot\frac{1}{\delta_3}-2, \frac{2}{\Delta} - 1  \right\} = \max\left\{ \frac{2-\alpha}{\alpha\delta_3}-2, \frac{4}{\alpha} - 1  \right\} \leq \frac{4}{\alpha\delta_3} \,.
\end{equation*}
Therefore, $m_3 \geq 4/({\alpha\delta_3})$ guarantees \eqref{eq::target} and $k_{\min}(\alpha, \delta_3, m_3) \in \mathcal{K}(\alpha, \delta_3, m_3)$. 
This completes the proof. 
\end{proof}

\begin{proof}[Proof of Lemma \ref{prop::R0}]
Introduce shorthand notation let $A=A_{\alpha,\delta_3}(m_3)$ (defined in Proposition \ref{prop::kmin}) and $\alpha_1=(m_3+1-k_{\min})/(m_3+1)$ for simplicity of exposition. 
For any $B_1,B_2 \in \R^+$, we have
\begin{eqnarray*}
\{ 
|R_0( \hat{\phi} ) 
- \alpha | > B_1 + B_2 \}
\,\subset\,
\{ 
| R_0( \hat{\phi} ) 
- \alpha_1| > B_1\}
\cup
\left\{ 
\left| \alpha_1
- \alpha \right| > B_2 \right\},
\end{eqnarray*}
and thus
\begin{align}
&\p\{
|R_0( \hat{\phi} ) 
- \alpha| > B_1 + B_2 \mid \hat \Lr\} \nonumber\\
&\leq\,
\p\{
| R_0( \hat{\phi} ) 
- \alpha_1| > B_1 \mid \hat \Lr\}
+
\p\left(
\left| \alpha_1
- \alpha \right| > B_2\mid \hat \Lr\right) \nonumber\\
%%%%%%%%%%%%%%%%%%%%%%%%%%%%
&\leq\, 
\frac{\kHat(m_3+1 - \kHat)}{(m_3+2)(m_3+1)^2}B_1 ^{-2} +
\1\left\{
\left| \alpha_1
- \alpha \right| > B_2 \right\},
\label{eq::distBound}
\end{align}
%\\
%%%%%%%%%%%%%%%%%
%\p\left\{
%|R_0( \hat{\phi} ) 
%- \alpha| > B_1 + B_2 \mid \hat \Lr\right\}
%&\leq&
%\p\left\{
%\left| R_0( \hat{\phi} ) 
%- \alpha_1 \right| > B_1 \mid \hat \Lr \right\}
%+
%\p\left(
%\left| \alpha_1
%- \alpha \right| > B_2 \mid \hat \Lr\right) \nonumber\\
%%%%%%%%%%%%%%%%%%%%%%%%%%%%%
%&\leq&
%\frac{\kHat(m_3+1 - \kHat)}{(m_3+2)(m_3+1)^2} B_1 ^{-2} +
%\1\left\{
%\left| \alpha_1
%- \alpha \right| > B_2 \right\},
%\label{eq::distBound}
%\end{eqnarray*}
where the last inequality follows from applying Lemma \ref{lem::Beta} 
to $R_0(\hat{\phi} )$ which follows $\text{Beta}(m_3+1-\kHat, \kHat)$ for $m_3\geq 4/(\alpha\delta_3)$ and continuous $F_{\hat r}$ due to Lemma \ref{lem::new}.
%%%%%%%%%%%%%%%%%%%%%%%%%%%%%%%%%%%%%%%%%%%%%%%%%%%%%%%%%
%%%%%%%%%%%%%%%%%%%%%%%%%%%%%%%%%%%%%%%%%%%%%%%%%%%%%%%%%
% proof: |R0-alpha| finite		proof: |R0-alpha| finite
%%%%%%%%%%%%%%%%%%%%%%%%%%%%%%%%%%%%%%%%%%%%%%%%%%%%%%%%%
It follows from Proposition \ref{prop::kmin} that
\begin{equation}
\label{eq::limA}
\left| \alpha - \alpha_1  \right| 
\,\leq\, A - (1-\alpha) + \frac{1}{m_3+1}\,.
\end{equation}
Letting $B_1 = \sqrt{\frac{\kHat(m_3+1-\kHat)}{(m_3+2)(m_3+1)^2\delta_4}}$ and $B_2 = A - (1-\alpha)+\frac{1}{m_3+1}$ in \eqref{eq::distBound} yields
\begin{align*}
&\p\{
|R_0( \hat{\phi} ) 
- \alpha| > \xi_{\alpha, \delta_3,m_3}(\delta_4) \mid \hat \Lr\}\\
&\leq\,
\delta_4
+
\1\{
\left| \alpha_1
- \alpha \right| > A - (1-\alpha)+\frac{1}{m_3+1}\}\\
&=\, \delta_4
\end{align*}
%\begin{equation*}
%\p\left( 
%\left|R_0( \hat{\phi} ) 
%- \alpha \right| > \xi_{\alpha, \delta_3,m_3}(\delta_4) \mid \hat \Lr\right)\\
%\,\leq\,
%\delta_4
%+
%\1\left\{
%\left| \alpha_1
%- \alpha \right| > A - (1-\alpha)+\frac{1}{m_3+1} \right\}\\
%\,=\, \delta_4
%\end{equation*}
for any arbitrary $\hat{r}$.
This, together with the independence between $\mathcal{S}^0_3$ and $\hat r$ (as a function of $(\mathcal{S}^0_1, \mathcal{S}^1_1, \mathcal{S}^0_2, \mathcal{S}^1_2)$) yields
\begin{eqnarray*}
\p\{
|R_0( \hat{\phi} ) 
- \alpha | >\xi_{\alpha, \delta_3,m_3}(\delta_4) \}
\,\leq\,\delta_4\,.
\end{eqnarray*}
To establish an upper bound for $\xi_{\alpha, \delta_3,m_3}(\delta_4)$, note that  
\begin{align*}
&\xi_{\alpha, \delta_3,m_3}(\delta_4)\\
%&=& \sqrt{\frac{\kHat(m_3+1-\kHat)}{(m_3+2)(m_3+1)^2\delta_4}} + A_{\alpha,\delta_3}(m_3) - (1-\alpha) + \frac{1}{m_3+1}\\
&= \sqrt{\frac{\kHat(m_3+1-\kHat)}{(m_3+2)(m_3+1)^2\delta_4}} + \frac{-1+2\alpha + \sqrt{1+4\delta_3(1-\alpha)\alpha(m_3+2)}}{2\left\{\delta_3(m_3+2)+1 \right\}}+ \frac{1}{m_3+1}\\
%%%%%%%%%%%%%%%%
&\leq \sqrt{\frac{(m_3+1)^2/4}{(m_3+2)(m_3+1)^2\delta_4}} 
+ \frac{1}{2\left\{\delta_3(m_3+2)+1 \right\}}
+ \frac{\sqrt{1+\delta_3(m_3+2)}}{2\left\{\delta_3(m_3+2)+1 \right\}}+ \frac{1}{m_3+1}\\
%%%%%%%%%%%%%%%
%&= \frac{1}{2\sqrt{(m_3+2)\delta_4}} 
%+ \frac{1}{2\left[\delta_3(m_3+2)+1 \right]}
%+ \frac{1}{2\sqrt{\delta_3(m_3+2)+1}}+ \frac{1}{m_3+1}\\
&< \frac{1}{2\sqrt{m_3\delta_4}} 
+ \frac{1}{2m_3\delta_3}
+ \frac{1}{2\sqrt{m_3\delta_3}}+ \frac{1}{m_3}\,.
\end{align*}
When $m_3 \geq \max(\delta_3^{-2}, \delta_4^{-2})$, we have
\begin{align*}
\xi_{\alpha, \delta_3,m_3}(\delta_4) 
&<\, \frac{1}{2m_3^{1/4}} + \frac{1}{2m_3^{1/2}} + \frac{1}{2m_3^{1/4}}+ \frac{1}{m_3}\\
&=\,  \frac{1}{m_3^{1/4}}\left( 1 + \frac{1}{2m_3^{1/4}} + \frac{1}{m_3^{3/4}} \right)
<  \frac{5/2}{ m_3^{1/4}} \,=\,\left( \frac{2}{5} m_3^{1/4}\right)^{-1}.
\end{align*}
This completes the proof. 
\end{proof}

\begin{proof}[Proof of Proposition \ref{prop::R1}]
 
Let $G^*=\{r <C_{\alpha}\}$ and $\widehat{G}=\{\hat r <\widehat{C}_{\alpha}\},$ 
the excess type II error can be decomposed as:
\begin{align}
\label{eq::decomposition}
&P_1(\widehat{G}) - P_1(G^*) \nonumber\\
&= \int_{\widehat{G}} d P_1 - \int_{G^*} d P_1 = \int_{\widehat{G}} \frac{p}{q}\, dP_0 - \int_{G^*} \frac{p}{q} \, d P_0\nonumber\\
&= \int_{\widehat{G}}( r - C_{\alpha}) dP_0 + C_{\alpha}P_0(\widehat{G}) 
- \int_{G^*} ( r-C_{\alpha})  d P_0 - C_{\alpha}P_0(G^*)\nonumber\\
&= \int_{\widehat{G}  \backslash G^*} (r - C_{\alpha}) dP_0 - \int_{G^* \backslash \widehat{G}}(r - C_{\alpha}) d P_0 + C_{\alpha}\{P_0(\widehat{G}) - P_0(G^*)\}\nonumber\\
&= \int_{\widehat{G} \backslash G^*} \left| r - C_{\alpha}\right| dP_0 + \int_{G^* \backslash \widehat{G}} \left| r - C_{\alpha}\right| d P_0+ C_{\alpha}\{R_0(\phiStar) - R_0(\hat{\phi})\}\,.
\end{align}
%%%%%%%%%%%%%%%%%%%%%%%%%%%
%%%%%%%%%%%%%%%%%%%%%%%%%%%
It follows from Lemma \ref{prop::R0} that 
when $m_3 \geq \max\{ \frac{4}{\alpha\delta_3}, \delta_3^{-2}, \delta_4^{-2}, (\frac{2}{5} M_1{\delta^*}^{\uderbar\gamma})^{-4}\}$,
\begin{eqnarray*}
\xi_{\alpha,\delta_3,m_3}(\delta_4)  \leq \frac{5}{2}m_3^{-1/4} \leq {M_1}(\delta^*)^{\uderbar\gamma}\,, \quad
\left\{\frac{\xi_{\alpha,\delta_3,m_3}(\delta_4)}{M_1}\right\}^{1/\uderbar\gamma} \leq \delta^*\,.
\end{eqnarray*}
Introduce shorthand notations $\Delta R_0 = |R_0(\phiStar) - R_0(\hat{\phi})|$, $\mathcal{E}_0 = \{\Delta R_0 < \xi_{\alpha,\delta_3,m_3}(\delta_4)\}$, and $T = \|\hat\Lr - \Lr\|_{\infty}$. 
On the event $\mathcal{E}_0$, 
\begin{eqnarray*}
\left(\frac{\Delta R_0}{M_1}\right)^{1/\uderbar\gamma} 
\,\le\, \left\{\frac{\xi_{\alpha,\delta_3,m_3}(\delta_4)}{M_1}\right\}^{1/\uderbar\gamma} \,\le\,  \delta^*\,.
\end{eqnarray*}
By the detection condition, we have 
\begin{eqnarray*}
\Delta R_0
\,\leq\, P_0\{C_{\alpha}<r(X)<C_{\alpha}+(\Delta R_0/M_1)^{1/\uderbar\gamma}\} \,.
\end{eqnarray*}
Note that 
\begin{align*}
P_0\{r(X)\geq C_{\alpha}+(\Delta R_0/M_1)^{1/\uderbar\gamma}\} 
%&=&
%P_0\{r(X) \geq C_{\alpha} \} - P_0\{C_{\alpha}<r(X)<C_{\alpha}+(\Delta R_0/M_1)^{1/\uderbar\gamma} \}\\
&=\,
R_0(\phi^*)  - P_0\{C_{\alpha}<r(X)<C_{\alpha}+(\Delta R_0/M_1)^{1/\uderbar\gamma} \}\\
&\leq\, R_0(\phi^*) - \Delta R_0\\  
&\leq\, R_0(\hat{\phi})  \,\,=\,\,  P_0\{\hat \Lr(X) > \widehat{C}_{\alpha}\}  \\
&\leq\, P_0\{ r(X) + T \ge \widehat{C}_{\alpha}\}\,\,=\,\, P_0\{r(X)\geq \widehat{C}_\alpha-T\}\,.
\end{align*}
Thus, we have $\widehat{C}_\alpha 
\leq C_{\alpha}+(\Delta R_0 /M_1)^{1/\uderbar\gamma} + T$, and 
\begin{align*}
\widehat{G} \backslash G^*  
&= \,\{ \Lr\geq C_{\alpha}, {\hat \Lr} < \widehat{C}_{\alpha}\} 
=  \{ \Lr\geq C_{\alpha}, {\hat \Lr} < 
{C}_{\alpha}+(\Delta R_0 /M_1)^{1/\uderbar\gamma} + T\} \cap \{{\hat \Lr} < \widehat{C}_{\alpha}\}\\
=&\,\,\{ 
{C}_{\alpha}+(\Delta R_0 /M_1)^{1/\uderbar\gamma} + 2T \geq \Lr \geq C_{\alpha}, {\hat \Lr} < {C}_{\alpha}+(\Delta R_0 /M_1)^{1/\uderbar\gamma} + T\}\cap \{{\hat \Lr} < \widehat{C}_{\alpha}\}\\
\subset&\,\, \{ {C}_{\alpha}+(\Delta R_0 /M_1)^{1/\uderbar\gamma} + 2T \geq \Lr\geq C_{\alpha}\}\,.
\end{align*}
Therefore, the margin assumption implies
\begin{align*}
P_0(\widehat{G} \backslash G^*) \,\,
&\leq \,P_0 \{ {C}_{\alpha}+(\Delta R_0 /M_1)^{1/\uderbar\gamma} + 2T \geq \Lr\geq C_{\alpha} \} \\
&\leq \,M_0  \{ (\Delta R_0 /M_1)^{1/\uderbar\gamma} + 2T\}^{\bar\gamma}\,.
\end{align*}
Hence on the event $\mathcal{E}_0$,
\begin{align*}
\int_{\widehat{G} \backslash G^*} \left| r- C_{\alpha}\right|{d}P_0 \,\,
&\leq \,\{(\Delta R_0 /M_1)^{1/\uderbar\gamma} + 2T\} P_0(\widehat{G} \backslash G^*) \\
&\leq\, M_0 \{ (\Delta R_0 /M_1)^{1/\uderbar\gamma} + 2T\}^{1+\bar\gamma}\,.
\end{align*}
%%%%%%%%%%%%%%%%%%%%%%%%%%%%%%%%%%
%%%%%%%%%%%%%%%%%%%%%%%%%%%%%%%%%%

We will bound $\int_{G^*\backslash \widehat{G}} \left| r- C_{\alpha}\right|{d}P_0$ on the event $\mathcal{E}_1= \{R_0(\hat \phi)\leq \alpha\}$.  Note that 
$$
P_0(r\geq C_{\alpha}) = \alpha \geq R_0(\hat \phi) =  P_0(\hat r \geq \widehat C_{\alpha})\geq  P_0(r \geq \widehat C_{\alpha}+\|\hat r - r\|_{\infty}) = P_0(r \geq \widehat C_{\alpha}+ T)\,.
$$
The above chain implies that $\widehat C_{\alpha} \geq C_{\alpha} - T$.  Therefore, 
\begin{align*}
G^*\backslash \widehat G &= \{r < C_{\alpha}, \hat r \geq \widehat C_{\alpha}\}\\
&= \{r < C_{\alpha}, r \geq r - \hat r + \widehat C_{\alpha}\}\\
&\subset \{r < C_{\alpha}, r\geq \widehat C_{\alpha} - T\}\\
& \subset \{C_{\alpha} - 2 T \leq r \leq C_{\alpha}\}\,.
\end{align*}
Hence on the event $\mathcal{E}_1$, 
$$
\int_{G^*\backslash \widehat{G}} \left| r- C_{\alpha}\right|{d}P_0 \leq 2T \cdot P_0(C_\alpha - 2T \leq r \leq C_{\alpha})
\leq M_0 (2T)^{1+\bar \gamma}\,,
$$
where the last inequality follows from the margin assumption.  
%Similarly, we can bound $\int_{G^*\backslash \hat{G}} \left| r- C_{\alpha}\right|{d}P_0$.  
Then it follows from 
\eqref{eq::decomposition} that on the event $\mathcal{E}_0 \cap \mathcal{E}_1$, 
\begin{align*}
\label{eq::upper-bound}
R_1(\hat{\phi}) - R_1({\phi}^*)\,\,
&\leq\, M_0  \left[ \left\{ \frac{ |\Delta R|}{M_1} \right\}^{1/\uderbar\gamma} + 2T\right]^{1+\bar\gamma} 
+ M_0(2T)^{1+\bar \gamma} + C_{\alpha} |R_0( \hat{\phi}) - R_0( \phi^*)| \\
&\leq\, 2M_0  \left[ \left\{ \frac{\xi_{\alpha,\delta_3,m_3}(\delta_4)}{M_1} \right\}^{1/\uderbar\gamma} + 2T\right]^{1+\bar\gamma} + C_{\alpha}\cdot \xi_{\alpha,\delta_3,m_3}(\delta_4)\,.
\end{align*}
From Lemma \ref{prop::R0}, we know that the event 
$\mathcal{E}_0$ occurs
with probability at least $1-\delta_4$. By Proposition \ref{prop::general k} and Proposition \ref{prop::kmin}， we know event $\mathcal{E}_1$ occurs with probability at least  $1-\delta_3$, so $\mathcal{E}_0\cap \mathcal{E}_1$ occurs with probability at least $1-\delta_3-\delta_4$.
This completes the proof.
\end{proof}

\begin{proof}[Proof of Proposition \ref{prop::exact recovery epsilon}]

%%%%%%%%%%%%%%%%%%%%%%%%%%%%%
Define event 
\begin{equation*}
\label{eq::Eecdf}
\mathcal{E}_{\delta_1} = 
\bigcap_{j=1}^d  \{\|\hat{F}^{1}_j-F^{1}_j \|_{\infty} < \delta_1^1 \}\cap  \{\|\hat{F}^{0}_j-F^{0}_j \|_{\infty} < \delta_1^0 \}\,,
\end{equation*}
where $\delta_1^1 =  \sqrt{\frac{\log(4d/\delta_1)}{2n_1}}$ and $\delta_1^0 = \sqrt{\frac{\log(4d/\delta_1)}{2m_1}}.$
On the event $\mathcal{E}_{\delta_1}$, for any $j \in \mathcal{A},$
\begin{align*}
\| \hat{F}^{0}_j - \hat{F}^{1}_j\|_{\infty} 
&\geq \| F^{0}_j - F^{1}_j \|_{\infty} - \| F_j^{0} - \hat{F}_j^{0}\|_{\infty} - \| F_j^{1} - \hat{F}_j^{1}\|_{\infty}\\
&\geq D - \| F_j^{0} - \hat{F}_j^{0}\|_{\infty} - \| F_j^{1} - \hat{F}_j^{1}\|_{\infty}\\
&> D - \delta_1^0 - \delta_1^1\,.
\end{align*}
For any $j \not\in \mathcal{A}$,
\begin{align*}
\| \hat{F}_j^{0} - \hat{F}_j^{1}\|_{\infty} 
&\leq  \| F_j^{0} - \hat{F}_j^{0}\|_{\infty} +\|F^{0}_j - F^{1}_j\|_{\infty}+ \| F_j^{1} - \hat{F}_j^{1}\|_{\infty} \\
&=  \| F_j^{0} - \hat{F}_j^{0}\|_{\infty} + \| F_j^{1} - \hat{F}_j^{1}\|_{\infty}\\
&<  \delta_1^0+ \delta_1^1\,.
\end{align*}

Since $n_1\geq 8D^{-2}\log(4d/\delta_1)$ and $m_1\geq 8D^{-2}\log(4d/\delta_1)$, 
$ \delta_1^0+ \delta_1^1 \leq D - \delta_1^0 - \delta_1^1$.
%%%%%%%%%%%%%%%%%%%%%%%%%%%%%%%%%%%%
%%%%%%%%%%%%%%%%%%%%%%%%%%%%%%%%%%%%
As a result,  on the event $\mathcal{E}_{\delta_1},$ any 
$\tau\in\left[ \delta_1^0+ \delta_1^1, D - \delta_1^0 - \delta_1^1 \right]$ would lead to 
$\widehat{\mathcal{A}}_{\tau}  =  \mathcal{A}$.
%%%%%%%%%%%%%%%%%%%%%%%%%%%%%%
%%%%%%%%%%%%%%%%%%%%%%%%%%%%%%
Therefore, 
\begin{align*}
\p( \widehat{\mathcal{A}}_{\tau}  =  \mathcal{A}) 
&\geq\, \p( \mathcal{E}_{\delta_1})\\
&\geq\, 1- \sum_{j=1}^d\left\{ \p(\|\hat{F}^{1}_j-F^{1}_j \|_{\infty} \geq \delta_1^1 )+ \p( \|\hat{F}^{0}_j-F^{0}_j \|_{\infty} \geq \delta_1^0)\right\}\\
& \geq\, 1-\delta_1\,,
\end{align*}
where the last inequality follows from applying Lemma \ref{lemma::DKW} to $F^0_j$ and $F^1_j$ for $j=1,\cdots, d$.
This completes the proof. 
\end{proof}

\begin{proof} [Proof of Proposition \ref{prop::joint_r}]

Define event 
$$
\mathcal{E} = \bigcap_{j \in \cA}\{ \|\log \hat p_j - \log p_j\|_{\infty} < B^1_j\}\cap \{\|\log \hat q_j - \log q_j\|_{\infty} < B^0_j\} \,,
$$
where 
$$B^1_j 
= \frac{C^1_j\sqrt{\frac{\log(2n_2 s/\delta_2)}{n_2{h_1}}}}{ \underline{\mu} - C^1_j\sqrt{\frac{\log(2n_2 s/\delta_2)}{n_2{h_1}}}},
\quad 
B^0_j 
= \frac{C^0_j\sqrt{\frac{\log(2m_2 s/\delta_2)}{m_2{h_0}}}}{ \underline{\mu} - C^0_j\sqrt{\frac{\log(2m_2 s/\delta_2)}{m_2{h_0}}}}\,.
$$

Let $B^1 = \sup_{j\in\mathcal{A}} B^1_j$ and $B^0 = \sup_{j\in\mathcal{A}} B^0_j$, we have $B \ge  s(B^0+B^1)$.
On the event $\{\cAhat_{\tau} = \cA\} \cap \mathcal{E},$ we have
$$
\log \hat \Lr^S_{\text{N}} (x) \, =\, \sum_{j \in \cAhat}\log\frac{\hat p_{j}(x_j)}{\hat q_{j}(x_j)} \,=\, \sum_{j \in \cA}\log\hat p_{j}(x_j)-\sum_{j \in \cA} \log\hat q_{j}(x_j)\,.
$$
Therefore, 
\begin{align*}
\| \log \hat \Lr^S_{\text{N}}- \log r \|_{\infty} \,\,
&=\,\, \| \sum_{j \in \cA} \log \hat p_{j} -  \sum_{j \in \cA} \log \hat q_{j} -  \sum_{j \in \cA} \log  p_{j} +  \sum_{j \in \cA} \log q_{j} \|_{\infty}\\
&\leq \,\, \sum_{j \in \cA} \left(\| \log \hat p_{j} - \log p_{j} \|_{\infty} + \| \log \hat q_{j} - \log q_{j} \|_{\infty}\right)\\
&\leq \,\,\sum_{j \in \cA} (B^1 + B^0)
\,\le\,  B\,.
\end{align*}
On the event $\{\cAhat_{\tau} = \cA\} \cap \mathcal{E}$, it follows from Lagrange's mean value theorem that
for any $x$,  there exists some $w_x$  between 
$\log \hat\Lr_{\text{N}}^S(x)$ and $\log r(x)$ such that
\begin{align*}
| \hat \Lr^S_{\text{N}}(x) - r(x) |\,
&=\, |e^{\log \hat\Lr_{\text{N}}^S(x)} - e^{\log r(x)}| 
\,=\, e^{w_x} | \log \hat\Lr_{\text{N}}^S(x) - \log r(x)|\\
&\leq\, e^{\|\log r\|_{\infty}+ B}B \,\,=\,\, Be^B \|r\|_{\infty} \,\,=\,\, T\,,
\end{align*}
where the last inequality  follows from the fact that 
$$w_x \leq \max( \log r(x), \log \hat\Lr_{\text{N}}^S(x)) \leq  \max( \|\log r\|_{\infty}, \|\log \hat\Lr_{\text{N}}^S\|_{\infty}) \leq \|\log r\|_{\infty}+ B\,.$$
Thus, 
$
\| \hat \Lr^S_{\text{N}} - r \|_{\infty}
\leq  T\,,
$
and we have
\begin{align}
\label{eq::Prob1+Prob2}
\p\left( \| \hat \Lr^S_{\text{N}} - r \|_{\infty}\leq T \right)\,\,
&\geq\, \p(\{\cAhat_{\tau} = \cA\}\cap \mathcal{E})
\,\,\geq\,\, \p(\cAhat_{\tau} = \cA) + \p(\mathcal{E}) - 1\nonumber\\
&=\, \p(\cAhat_{\tau} = \cA) -  \p(\mathcal{E}^c)\,.
\end{align}

\noindent By Proposition \ref{prop::exact recovery epsilon}, we have 
\begin{equation}
\label{eq::Prob1}
\p(\cAhat_{\tau} = \cA) \,\geq\, 1- \delta_1\,.
\end{equation}
Also, it follows from Lemma \ref{lemma::A1_1-dim} that
$$\p \left(\| \log \hat p_{j} - \log p_{j} \|_{\infty}> B^1_j \right)\vee \p \left(\| \log \hat q_{j} - \log q_{j} \|_{\infty}> B^0_j \right)
\,\,\le\,\, \delta_2/(2{s})\,.
$$
Therefore, 
\begin{equation}
\label{eq::Prob2}
\p( \mathcal{E}^c)
\,\,\le\,\,  (2s) \delta_2/(2s) \,\, =\,\, \delta_2\,.
\end{equation}
Plugging \eqref{eq::Prob1} and \eqref{eq::Prob2} back to \eqref{eq::Prob1+Prob2} yields \eqref{eq::T}. Moreover, because $s \leq n_2 \wedge m_2$, it follows from Lemma \ref{lemma::A1_1-dim} that there exists some $\bar C_2>0$, such that 
\begin{equation*}
B\,\,\leq\,\, \bar {C}_2  \, s\left\{ \left(\frac{\log n_2}{n_2}\right)^{\frac{\beta}{2\beta+1}} + \left(\frac{\log m_2}{m_2}\right)^{\frac{\beta}{2\beta+1}}\right\}.
\end{equation*}
Moreover, since $s \leq (n_2\wedge m_2)^{\frac{\beta}{2(\beta+1)}}$, the above bound implies that $B$ is bounded from above by some absolute constant. Also note that $\|r\|_{\infty}$ is bounded from above, so there exists an absolute constant $C_2>0$, such that 
$$T = B e^{B} \|r\|_{\infty}\leq {C}_2 \, s\left\{ \left(\frac{\log n_2}{n_2}\right)^{\frac{\beta}{2\beta+1}} + \left(\frac{\log m_2}{m_2}\right)^{\frac{\beta}{2\beta+1}}\right\}.$$
\noindent This completes the proof.
\end{proof}

\begin{proof}[Proof of Theorem \ref{theorem::2}]
Combining Propositions \ref{prop::general k},  \ref{prop::kmin}, \ref{prop::R1} and \ref{prop::joint_r}, \begin{eqnarray*}
\label{eq::thm_finale_1}
\p\left( R_0(\hat\phi_{\text{NSN}^2}) \leq \alpha, 
R_1(\hat\phi_{\text{NSN}^2}) \leq R_1(\phi^*) + W \right) \geq 1 - \delta_1-\delta_2 - \delta_3 - \delta_4\,,
\end{eqnarray*}
where 
\begin{align*}
W &=\,\, 2M_0\left[ \left(\frac{2}{5}m_3^{1/4}M_1\right)^{-1/\uderbar{\gamma}} + 2 C_2\, s\left\{\left(\frac{\log n_2}{n_2}\right)^{\frac{\beta}{2\beta+1}} + \left(\frac{\log m_2}{m_2}\right)^{\frac{\beta}{2\beta+1}}\right\} \right]^{1 + \bar\gamma}\nonumber \\ 
&\phantom{=}+ \,\, C_{\alpha} \left(\frac{2}{5}m_3^{1/4}\right)^{-1}.
\end{align*}
This completes the proof. 
\end{proof}

\begin{lem}[Dvoretzky-Kiefer-Wolfowitz inequality\citep{dvoretzky1956asymptotic}]
\label{lemma::DKW}
Let $X_1$, $X_2$, $\cdots$, $X_n$ be real-valued i.i.d. random variables with \textsc{cdf} $F(\cdot)$, and let $\hat{F}_n(x) = n^{-1}\sum_{i=1}^n \1(X_i \leq x )$.
For any  $t > 0$,
it holds that
\begin{equation*}
\label{eq::1}
\p(\|\hat{F}_n-F\|_{\infty}\geq t)\,\,\leq\,\, 2e^{-2nt^2}\,.
\end{equation*}
Or,  for any given $\delta \in (0,1),$ 
\begin{equation}
\label{eq::2}
\p(\|\hat{F}_n-F\|_{\infty}\geq \sqrt{\frac{\log(2/\delta)}{2n}}) \,\,\leq\,\, \delta\,.
\end{equation}
\end{lem}

\begin{lem}\label{lemma::A1_1-dim}
Given a density function $p\in\mathcal{P}_{\Sigma}(\beta, L, [-1,1]),$ 
construct its kernel estimate $\hat{p}(x)=(n{h})^{-1}\sum_{i=1}^n K\left(\frac{X_i-x}{h}\right)$ from i.i.d. sample $\{X_i\}_{i=1}^n,$
where the kernel $K$ is $\beta$-valid and $L'$-Lipschitz, and the bandwidth $h = \left(\log n/{n}\right)^{\frac{1}{2\beta+1}}$.  
For any  $\delta\in(0, 1)$, as long as the sample size $n$ is such that $\sqrt{\frac{\log(n/\delta)}{n{h}}}< \min(1, \underline{\mu}/C),$
where
$C=\sqrt{48c_1} + 32c_2+2Lc_3+L'+L+\tilde C\sum_{1\leq|l|\leq\lfloor\beta\rfloor}\frac{1}{l!}$, in which $c_1=\|p\|_{\infty}\|K\|_2^2$, $c_2=\|K\|_{\infty}+\|p\|_{\infty}+\int|K||t|^{\beta}dt$, $c_3=\int |K||t|^{\beta}dt$, and $\tilde C$ is such that $\tilde C \geq \sup_{1\leq|l|\leq\lfloor \beta\rfloor}\sup_{x\in[-1, 1]}|p^{(l)}(x)|$, and  $\underline{\mu} (> 0)$ is a lower bound of $p$,
we have %\textcolor{red}{make $c_1..$ explicit}
\begin{equation}
\label{eq::2}
\p\left(\|\log \hat{p}- \log p\|_{\infty}\geq U \right) \,\,\leq\,\, \delta\,,
\end{equation}
where 
$U= \frac{C\sqrt{\frac{\log(n/\delta)}{n{h}}}}{\underline{\mu} - C\sqrt{\frac{\log(n/\delta)}{n{h}}}}$. 
When $n\geq 1/\delta$, we have $U\leq C_1\left(\log n/{n}\right)^{\frac{\beta}{2\beta+1}}$ for some absolute constant $C_1$.
\end{lem}
\begin{proof}
Let $\mathcal{E}_1 = \{ \|\hat{p}-p\|_{\infty}\leq C\sqrt{\frac{\log(n/\delta)}{n{h}}} \}$. On the event $\mathcal{E}_1$, since $\sqrt{\frac{\log(n/\delta)}{n{h}}}< \min(1, \underline{\mu}/C)$, we have 
$$\min(p(x_0), \hat p(x_0) ) \geq  \min(p(x_0), p(x_0)  -\|\hat{p}-p\|_{\infty}) \geq  \underline{\mu} - \|\hat{p}-p\|_{\infty} >0\,.$$
It then follows from Lagrange's mean value theorem that for any fixed $x_0,$ 
there exists some $w_{x_0}$ between $\hat p(x_0)$ and $p(x_0)$, 
\begin{align*}
| \log \hat p(x_0) - \log p(x_0) | 
&=\,w^{-1}_{x_0}\, |\hat p(x_0) - p(x_0)|\\
&\leq\, [\min\{\hat p(x_0), p(x_0)\}]^{-1}   |\hat p(x_0) - p(x_0)|
\leq \frac{ \|\hat{p}-p\|_{\infty}}{\underline{\mu} - \|\hat{p}-p\|_{\infty}}.
\end{align*} 
As a result, it holds on event $\mathcal{E}_1$ that
$$
\| \log \hat p - \log p \|_{\infty}
\,\,\leq\,\,
\frac{C\sqrt{\frac{\log(n/\delta)}{n{h}}}}{\underline{\mu} - C\sqrt{\frac{\log(n/\delta)}{n{h}}}} \,\,=\,\, U\,,
$$
and 
$$
\p(\| \log \hat p - \log p \|_{\infty}
\leq U )
\geq
\p(\|\hat{p}-p\|_{\infty}\leq C\sqrt{\frac{\log(n/\delta)}{n{h}}} ) 
\geq 
1-\delta\,,
$$
where the last inequality follows from Lemma A.1 in \cite{Tong.2013} (the special case of $d=1$).
Finally when $n\geq 1/\delta$, 
we have $U= \frac{C\sqrt{\frac{\log(n/\delta)}{n{h}}}}{\underline{\mu} - C\sqrt{\frac{\log(n/\delta)}{n{h}}}}\leq C_1\left({\log n}/{n}\right)^{\frac{\beta}{2\beta+1}}$ for some absolute constant $C_1$.
This completes the proof.
\end{proof}

%%%%%%%%%%%%%%%%%%%%%%%%%%%%%%%%%%%%%%%%%%%%%%%
%%%%%%%%%%%%%%%%%%%%%%%%%%%%%%%%%%%%%%%%%%%%%%%%%%%%%%%%%
% proof: |R0-alpha| asymptotic		proof: |R0-alpha| asymptotic
%%%%%%%%%%%%%%%%%%%%%%%%%%%%%%%%%%%%%%%%%%%%%%%%%%%%%%%%%
%On the other hand, 
%setting $B_1 = B_2 = \tau=1$ in \eqref{eq::distBound} gives us:
%\begin{eqnarray*}
%0 &\leq& \p\left( \left.
%\left|R_0( \hat{\phi} ) 
%- \alpha \right| > 2\tau \right| \hat \Lr \right)\\
%&\leq&
%\frac{\kHat(m_3+1 - \kHat)}{(m_3+2)(m_3+1)^2}\cdot\tau ^{-2}
%+ \1\left\{\left| \alpha - \alpha_1  \right| > \tau \right\}.
%\end{eqnarray*}
%Let $m_3$ go to infinity on both sides of the above equation we have
%\begin{eqnarray*}
%0 &\leq& \lim_{m_3\to\infty} 
%\p\left( \left.
%\left|R_0( \hat{\phi}  ) 
%- \alpha \right| > 2\tau  \right| \hat \Lr \right) \\
%&\leq& 
%\lim_{m_3\to\infty} \frac{1}{m_3+2}\cdot\frac{\kHat}{m_3+1}\cdot\alpha_1\cdot\tau ^{-2}\\
%&&+ \lim_{m_3\to\infty}\1\left\{ \left| \alpha - \alpha_1  \right| > \tau \right\} = 0,
%\end{eqnarray*}
%where the last ``=" follows from \eqref{eq::limA}.
%Thus, for any $\tau > 0.$
%\begin{eqnarray*}
%\lim_{m_3\to\infty} 
%\p\left( \left.
%\left|R_0( \hat{\phi}  ) 
%- \alpha \right| > 2\tau  \right| \hat \Lr  \right) = 0.
%\end{eqnarray*}

\section{About detection condition and    Assumption \ref{assumption::beta-valid}}\label{sec::assumption3 and detection condition}
%First, we continue our discussion of necessity of detection condition. In the main text, the first case of \eqref{eq::detection-not-hold} was covered, and here  we would like to discuss the implication of the second case. If $P_0\{ C_{\alpha}-\delta \leq r(X) \leq C_{\alpha}\}=0$, one cannot differentiate between $C_{\alpha}$ and $C_{\alpha}-\delta$ for the estimate $\hat C_{\alpha}$. Then $\hat C_{\alpha}<C_{\alpha}-\delta/2$ with non-vanishing probability. 
%Note that since we assume that $r$ is known ($\hat r = r$), 
%\begin{align*}
%	G^*\setminus \hat G 
%	& = \{r < C_{\alpha}, \hat r \geq \hat C_{\alpha}\}\\
%	& = \{r < C_{\alpha},  r \geq \hat C_{\alpha}\}\\
%	&\supset \{r<C_{\alpha}, r>C_{\alpha}-\delta/2\}\\
%	&\supset \{C_{\alpha}-\delta/2+T<\textcolor{red}{r<C_{\alpha}-\delta/8}\}.
%\end{align*}
%The excess type II error can be decomposed according to \eqref{eq::decomposition}:
%\begin{align*}
%%\label{eq::decomposition}
%P_1(\widehat{G}) - P_1(G^*) 
%= \int_{\widehat{G} \backslash G^*} \left| r - C_{\alpha}\right| dP_0 + \int_{G^* \backslash \widehat{G}} \left| r - C_{\alpha}\right| d P_0+ C_{\alpha}\{R_0(\phiStar) - R_0(\hat{\phi})\}\,.
%\end{align*}
%
%
%As a result, $\int_{G^*\backslash \hat{G}} \left| r- C_{\alpha}\right|{d}P_0>P_0(G^*\setminus\hat G)\delta/8>0$. Therefore, the excess type II error cannot diminish as $n$ diverges since the other two terms in\eqref{eq::decomposition}  are nonnegative asymptotically. This indicates that both sides of the detection condition are necessarily. 

We show that it is possible for densities satisfying Assumption \ref{assumption::beta-valid} to violate a generalized version of the detection condition defined in Definition \ref{def::detection}.
While the generalized detection condition applies to general $(P, f, C^*)$ as the original one, 
we narrow its definition to  $(P_0,r,C_{\alpha})$ which we actually use in the main text.

\begin{defin}[Generalized detection condition]
Let $u(\cdot)$ be a strictly increasing differentiable function on $\mathbb{R}^+$ with $\lim_{x\rightarrow 0+} u(x) = 0$,
a function $r(\cdot)$ is said to satisfy the \emph{generalized detection condition} with respect to $P_0$ and $u(\cdot)$ at level $(C_\alpha, \delta^*)$ if for any $\delta \in (0, \delta^*)$, 
\begin{eqnarray}\label{assumption:: detection general}
P_0\left\{C_{\alpha} \leq r(X) \leq C_{\alpha} + \delta\right\}&\geq& u(\delta)\,.
\end{eqnarray}
\end{defin}

\medskip
The following conditions suffice to make \eqref{assumption:: detection general} fail 
\begin{eqnarray}\label{eqn::sufficient}
P_0\left\{ C_{\alpha} \,\leq\, r(X) \,\leq\, C_{\alpha}+ k^{-1} \right\} 
&<& u(k^{-1})\,, \quad k = 1, 2, \ldots \,.
\end{eqnarray}
%Let $G = \{x: r(x) \leq C_\alpha\}$ and $G_{k} = \{x: r(x) < C_\alpha - k^{-1}\}$ $(k = 1, 2, \ldots)$,
%the left-hand side of \eqref{eqn::sufficient} can be written as
%\begin{eqnarray*}
%P_0\left\{ C_\alpha - k^{-1}\,\leq\, r(X) \,\leq\, C_\alpha  \right\} 
%\,\,=\,\,
%P_0\left\{  X \in G / G_k  \right\}\,\,=\,\,
%\int_{G / G_k} q(x) dx \label{eq::int}\,.
%\end{eqnarray*}
%To construct an example that satisfies \eqref{eqn::sufficient} is essentially to find a sequence of nested sub-regions $G_k$ within $G$ such that 
%\begin{eqnarray}
%\label{eq::33}
%G_1 \,\subset\, G_2 \ldots \,\subset\, G_k \,\subset\, \ldots \,\subset\, G\,,\quad 
%\int_{G / G_k} q(x) dx  \,\,<\,\, g(k^{-1})\,.
%\end{eqnarray}
%%%%%%%%%%%%%%%%%%%%%
A 1-dimensional toy example that satisfies Assumption \ref{assumption::beta-valid}  and \eqref{eqn::sufficient} (thus violating the generalized detection condition) is given as follows. 
Assume $P_0$ and $P_1$ have the same support $[-1,1]$. 
Given $u(\cdot)$ as a strictly increasing differentiable function on $\mathbb{R}^+$ with $ \lim_{x\rightarrow 0+} u(x) = 0$, 
let $q(x) = \alpha$ for all $x \in [0,1]$, and set $p(x)$ accordingly such that
\begin{eqnarray} 
\label{eq::constructR}
r(x) &=& \frac{p(x)}{q(x)} \,\,=\,\,
\left\{
\begin{array}{ll}
2u^{-1}(1)  + 2 u^{-1}(\alpha x)\,, &x \in (0, 1]\,,\\
2u^{-1}(1)\,, & x = 0\,,\\
2u^{-1}(1) - v(x)\,, & x \in [-1, 0)\,,
\end{array}
\right.
\end{eqnarray}
where $v(\cdot)$ is some  positive differentiable function that makes $r(\cdot) $ differentiable at $x=0$. 
%Given $\alpha x \in [-1, 1]$ for $\forall x \in [-1,1]$, we have  
%\begin{eqnarray*}
%r(x) &\in & \left[ 2 + 2u^{-1}(-1)\,, 2 + 2u^{-1}(1) \right] \,\,=\,\,[0, 4]
%\end{eqnarray*}
%is bounded.
It follows from \eqref{eq::constructR} that 
$\left\{x \in [-1,1]: r(x) \geq 2u^{-1}(1) \right\} \,=\, [0,1]$, 
and identity 
\begin{eqnarray*}
P_0\left\{r(X) \geq 2u^{-1}(1) \right\} 
\,\,=\,\,
\int_{\left\{x \in [-1,1]:\, r(x) \geq 2u^{-1}(1) \right\}} q(x) dx 
\,\,=\,\, \int_{[0,1]} q(x) dx \,\,=\,\, \alpha
\end{eqnarray*}
implies
$C_\alpha = 2u^{-1}(1)$.
As a result, for any $k \in \{1, 2, \ldots\}$ we have
\begin{eqnarray*}
\left\{ C_\alpha \leq r(X) \leq C_\alpha + k^{-1}\right\}
\,\,=\,\,  \left\{ X \in [0,1],\, 2u^{-1}(\alpha X) \leq  k^{-1} \right\}
\,\,=\,\, \left\{ X \in \left[0,  \alpha^{-1} u(0.5 k^{-1})\right] \right\},
\end{eqnarray*}
and
\begin{eqnarray*}
P_0\left\{ C_\alpha \,\leq\, r(X) \,\leq\, C_\alpha + k^{-1} \right\} 
&=&
P_0\left\{ X \in \left[0,  \alpha^{-1} u(0.5 k^{-1})\right] \right\}
\,\,=\,\, \int_0^{\alpha^{-1} u(0.5 k^{-1})} q(x) dx \\
&=&
\alpha\cdot \alpha^{-1} u(0.5 k^{-1}) \,\,=\,\, u\left(0.5k^{-1} \right)  \,\, < \,\, u\left( k^{-1} \right)
\end{eqnarray*} 
satisfies \eqref{eqn::sufficient}.
Note that the above construction makes no assumption about the behavior of $q(\cdot)$ and $p(\cdot)$ on $[-1,0)$ except the normalization constraints $\int_{[-1,1]}p dx = \int_{[-1,1]}q dx = 1$ and $r(\cdot)$ being differentiable on $[-1,1]$.
Thus, there exist $p$, $q$, and $r$ that satisfy Assumption \ref{assumption::beta-valid}.

\section{An alternative threshold estimate}\label{sec::alternative threshold}
This part contains an alternative estimate of threshold $C_{\alpha}$ that guarantees type I error bound. Based on Chernoff inequality,  the following Proposition gives an alternative version  of Proposition \ref{prop::general k}. First, we introduce two technical lemmas.
%%%%%%%%%%%%%%%%%%%%%%%%%%%%%%%%%%%%%%%%%%%%%%%%%%%%%%%%%%%%%%%%%%%
%	GAMMA BOUND
%%%%%%%%%%%%%%%%%%%%%%%%%%%%%%%%%%%%%%%%%%%%%%%%%%%%%%%%%%%%%%%%%%%
\begin{lem}\label{lemma::gamma-deviation} 
If $G_k\sim \text{Gamma}(k,1),$ $k > 0,$
then for any  $\tau \in (0,k)$, we have
\begin{eqnarray*}
\p( G_k \geq k + \tau ) \,\,\leq\,\, e^{-\tau^2/(4k)}, 
\quad \p( G_k \leq k - \tau ) \,\,\leq\,\, e^{-\tau^2/(2k)}\,\,\leq\,\, e^{-\tau^2/(4k)}.
\end{eqnarray*}
\end{lem}
%%%%%%%%%%%%%%%%%%%%%%%%%%%%%%%%%%%%%%%%%%%
\begin{proof}[Proof of Lemma \ref{lemma::gamma-deviation}] For any $\epsilon  \in (0, 1)$ and $t \in (0,1)$, 
 it follows from Chernoff inequality that
\begin{eqnarray}
\label{eq::chernoff+}
\p \left\{ G_k \geq (1+\epsilon) k \right\}  
&=&
\p \left\{ e^{tG_k} \geq e^{t(1+\epsilon) k}\right\} 
\,\,\leq\,\,
\frac{\E(e^{tG_k})}{e^{t(1+\epsilon) k}}\,\,=\,\,
(1-t)^{-k} e^{-t(1+\epsilon)k}.
\end{eqnarray}
Letting
$t = \argmin_{x\in (0,1)}(1-x)^{-k} e^{-x(1+\epsilon)k} = \epsilon/(1+\epsilon)$ in \eqref{eq::chernoff+}
yields
\begin{eqnarray*}
\p \left\{ G_k \geq (1+\epsilon) k \right\}   \,\,\leq\,\, (1+\epsilon)^k e^{-\epsilon k}
\,\,=\,\, e^{k \left\{\log(1+\epsilon) - \epsilon\right\}}
\,\,\leq\,\, e^{-k \epsilon^2/4},
\end{eqnarray*}
%%%%%%%%%%%%%%%%%%%%%
%Let 
%$e_1(t) = \log(1-t) + (1+\epsilon)t$, 
%we have
%\begin{eqnarray*}
%e_1^\prime(t) \,\,=\,\, - \frac{1}{1-t}+1+\epsilon ,
%\quad
%e_1^{\prime\prime}(t) \,\,=\,\, - \frac{1}{(1-t)^2} \,\,<\,\, 0\,. 
%\end{eqnarray*}
%Given $t_1^* = \epsilon/(1+\epsilon)\in(0,1)$ satisfy 
%\begin{eqnarray*}
%e_1^\prime(t_1^*) \,\,=\,\, 0\, \quad 
%e_1^{\prime\prime}(t_1^*) \,\,<\,\, 0\,, 
%\end{eqnarray*} 
%the maximum of $e_1(t)$ on $(0,1)$ is
%\begin{eqnarray*}
%e_1(t_1^*) 
%= - \log(1+\epsilon) + \epsilon 
%> \frac{\epsilon^2}{4}, \quad\forall 0 < \epsilon < 1\,.
%\end{eqnarray*}
%Thus, let 
%$t = t_1^*$ yields
%in (\ref{eq::chernoff+}),
%\begin{eqnarray*}
%\p ( G_k \geq (1+\epsilon) k)  
%\leq e^{-k [- \log(1+\epsilon) + \epsilon]}
%\leq e^{-k \epsilon^2/4}, \quad\forall 0 < \epsilon < 1\,.
%\end{eqnarray*}
%
%On the other hand, for any $t<0$,
Likewise, for any $\epsilon \in (0,1)$ and $s < 0$, 
\begin{eqnarray}
\label{eq::chernoff-}
\p\left\{ G_k \leq (1-\epsilon) k\right\}  
&=&
\p\left\{ e^{sG_k} \geq e^{s(1-\epsilon) k} \right\} 
\,\,=\,\,(1-s)^{-k} e^{-s(1-\epsilon)k}\,.
\end{eqnarray}
Letting 
$s = \argmin_{x < 0}(1-x)^{-k} e^{-x(1-\epsilon)k} = - \epsilon/(1-\epsilon)$ in \eqref{eq::chernoff-} yields 
\begin{eqnarray*}
\p \left\{ G_k \leq (1-\epsilon) k\right\} &\leq& (1-\epsilon)^k e^{\epsilon k}  \,\,=\,\, e^{k\left\{ \log(1-\epsilon) + \epsilon\right\}}
\,\,\leq\,\, e^{-k \epsilon^2/2}, 
\end{eqnarray*}
%%%%%%%%%%%%%%%%%%%%%%
%Let $e_2(t) = \log(1-t) + (1-\epsilon)t,$
%we have
%\begin{eqnarray*}
%e_2^\prime(t) = - \frac{1}{1-t}+1-\epsilon ,
%\quad
%e_2^{\prime\prime}(t) = - \frac{1}{(1-t)^2} < 0.
%\end{eqnarray*}
%Thus, $e_2(t)$ reaches its maximum at $t_2^* = - \frac{\epsilon}{1-\epsilon} < 0$ 
%with value
%\begin{eqnarray*}
% e_2(t_2^*) 
%= - \log(1-\epsilon) - \epsilon.
%\end{eqnarray*}
where the last inequality follows from Taylor expansion
\begin{eqnarray*}
\log(1-\epsilon) + \epsilon &=& \sum_{i=1}^{\infty} \frac{\epsilon^i}{i} -\epsilon \,\,=\,\, \sum_{i=2}^{\infty} \frac{\epsilon^i}{i} 
\,\,>\,\, \frac{\epsilon^2}{2}, \quad\forall 0 < \epsilon < 1\,.
\end{eqnarray*}
%Thus, let $t = t_2^*$ in \eqref{eq::chernoff-},
%\begin{eqnarray*}
%\p ( G_k \leq (1-\epsilon) k)  
%\leq e^{-k [- \log(1-\epsilon) - \epsilon]}
%\leq e^{-k \epsilon^2/2}, \quad\forall 0 < \epsilon < 1\,.
%\end{eqnarray*}
Take $\epsilon = \tau/k$, the conclusion of the lemma follows.

%%%%%%%%%%%%%%%%%%%%%%%%%%%%%%%%%%%%%%
%%%%%%%%%%%%%%%%%%%%%%%%%%%%%%%%%%%%%%
\iffalse
Jointly, 
\begin{eqnarray*}
\p ( |G_k - k| \geq \epsilon k \} 
&=& \p ( G_k \geq (1+\epsilon) k) + \p ( G_k \leq (1-\epsilon) k)\\ 
&\leq& e^{-k \epsilon^2/6} + e^{-k \epsilon^2/2}
\leq 2e^{-k \epsilon^2/6}\,.
\end{eqnarray*}
\fi

\end{proof}

%%%%%%%%%%%%%%%%%%%%%%%%%%%%%%%%%%%%%%%%%%%%%%%%%%%%%%%
%	BETA BOUND		BETA BOUND		BETA BOUND		BETA BOUND		BETA BOUND		BETA BOUND
%%%%%%%%%%%%%%%%%%%%%%%%%%%%%%%%%%%%%%%%%%%%%%%%%%%%%%%%%%
\begin{lem}
\label{lem::Beta-chernoff}
Let $B \sim \text{Beta}(a,b)$, and $\mu = \E(B) =a/(a+b)$. 
For any $t \in (0,1-\mu)$,
\begin{eqnarray*}
\p\left\{ B> \mu + t\right\}
\,\leq\,
2\exp\left[-4^{-1} \left\{\frac{(a+b)t}{\sqrt{b}(\mu+t)+\sqrt{a}(1-\mu-t)}\right\}^2 \right].
\end{eqnarray*}
\end{lem}

%%%%%%%%%%%%%%%%%%%%%%%%%%%%%%%%%%%%%%%%%%%%%
\begin{proof}[Proof of Lemma \ref{lem::Beta-chernoff}]
By properties of beta distribution, we can represent $B$ as
\begin{eqnarray*}
B&=& \frac{G_a}{G_a+G_b}\,,\quad\text{where $G_a\sim \Gamma(a,1)$, $G_b\sim \Gamma(b,1)$ are independent.}
\end{eqnarray*}
For any $t>0$ and constant $C$ such that $a(1-\mu-t)< C < b(\mu+t)$, we have 
\begin{eqnarray}
\label{eq::suff}
\p(B\leq \mu+t)&=& \p\left\{(1-\mu-t)G_a\leq (\mu+t)G_b\right\}
\,\,\geq\,\, \p\left\{(1-\mu-t)G_a\leq C \leq (\mu+t)G_b \right\}\nonumber\\
&=& \p\left\{(1-\mu-t)G_a\leq C \right\} \p\left\{ C \leq (\mu+t)G_b \right\}\nonumber\\
&=& \p\left(G_a\leq\frac{C}{1-\mu-t} \right)\p\left(G_b\geq \frac{C}{\mu+t} \right)\nonumber\\
&=&\left\{1-\p\left(G_a>\frac{C}{1-\mu-t} \right)\right\}\left\{1-\p\left(G_b> \frac{C}{\mu+t} \right)\right\}\nonumber\\
&\geq& 1- \p\left(G_a>\frac{C}{1-\mu-t} \right) - \p\left(G_b> \frac{C}{\mu+t} \right),
%%%%%%%%%%%%%%
\end{eqnarray}
where by Lemma \ref{lemma::gamma-deviation}
\begin{eqnarray}
&&\p\left(G_a > \frac{C}{1-\mu-t} \right)
\,\,\leq\,\, \p\left\{G_a > a+\left(\frac{C}{1-\mu-t}-a\right)\right\}
\,\,\leq\,\, 
%\exp\left\{-\left(\frac{C}{1-\mu-t}-a\right)^2 (4a)^{-1}\right\},\\
e^{-\left(\frac{C}{1-\mu-t}-a\right)^2 (4a)^{-1}},\nonumber\\
%%%%%%%%%%%%%%%%%%%
&&\p\left(G_b < \frac{C}{\mu+t} \right)
\,\,\leq\,\, \p\left\{G_b < b-\left(b- \frac{C}{\mu+t} \right)\right\}
\,\,\leq\,\, e^{-\left(b- \frac{C}{\mu+t} \right)^2(4b)^{-1}}\label{eq::ab}.
\end{eqnarray}
Letting
\begin{eqnarray*}
C&=&\frac{(1-\mu-t)(\mu+t)(a\sqrt{b}+\sqrt{a}b)}{\sqrt{b}(\mu+t)+\sqrt{a}(1-\mu-t)}
\end{eqnarray*}
in \eqref{eq::suff} such that the two exponents in \eqref{eq::ab} equal 
$$
\left(\frac{C}{1-\mu-t}-a\right)^2 (4a)^{-1}
\,\,=\,\, \left(b- \frac{C}{\mu+t} \right)^2(4b)^{-1}
\,\,=\,\, 4^{-1}\left\{\frac{(a+b)t}{\sqrt{b}(\mu+t)+\sqrt{a}(1-\mu-t)}\right\}^2
$$
yields
\begin{eqnarray*}
\p(B > \mu+t)&=&1-\p(B \leq \mu+t)
\,\,\leq\,\,\p\left(G_a>\frac{C}{1-\mu-t} \right) + \p\left(G_b> \frac{C}{\mu+t} \right)\\
&\leq& e^{-\left(\frac{C}{1-\mu-t}-a\right)^2 (4a)^{-1}} + 
e^{-\left(b- \frac{C}{\mu+t} \right)^2(4b)^{-1}}\\
&=& 2\exp\left[-4^{-1} \left\{\frac{(a+b)t}{\sqrt{b}(\mu+t)+\sqrt{a}(1-\mu-t)}\right\}^2 \right].
\end{eqnarray*}
This completes the proof.
\end{proof}

\begin{prop}
\label{prop::general k_chern}
Let $\hat{\Lr}(\cdot)$ be any estimate of the density ratio function. 
For any $\delta_3 \in (0,1)$ and $k\in \{1,\cdots, m_3\}$,  the type I error of classifier $\hat\phi_k$ defined in (2.1) satisfies
\begin{eqnarray*}
\label{eq::Chernoff-bound}
\p \left\{ R_0(\hat\phi_k) > h(\delta_3, m_3, k)
\right\} \,\leq\, \delta_3\,,
\end{eqnarray*}
where
$$h(\delta_3,m_3,k) = \frac{m_3+1-k + 2\sqrt{\log\left(2/\delta_3\right)}\sqrt{m_3-k+1}}
{m_3+1+ 2\sqrt{\log\left(2/\delta_3\right)}\left( \sqrt{m_3-k+1} - \sqrt{k}\right)}\,.$$
\end{prop}

\begin{proof}
Let $B$ be a realization from Beta$(k,m_3+1-k)$. 
It follows from Proposition \ref{prop::general delta} that
\begin{eqnarray*}
\p\{ R_0(\hat{\phi}_k)> h(\delta_3,m_3,k) \} 
&\leq& \text{Beta.cdf}_{k, m_3+1 -k}\left\{1-h(\delta_3,m_3,k)\right\}\\
&=& \p\{ B \leq  1 - h(\delta_3,m_3,k) \}\,\,=\,\, \p\{ 1- B \geq h(\delta_3,m_3,k)\}
\end{eqnarray*}
for any $k \in \{1,\cdots,m_3\}$ and $\hat r$, with $1-B \sim \text{Beta}(m_3+1-k,k)$. 
Letting $a = m_3+1-k$, $b =k$, and 
\begin{eqnarray*}
&& t = \frac{2\sqrt{\log(2/\delta_3)}\left\{(m_3+1-k)\sqrt{k}+k \sqrt{m_3+1-k} \right\}}
{(m_3+1)\left\{m_3+1+ 2\sqrt{\log(2/\delta_3)}\left(\sqrt{m_3+1-k}-\sqrt{k}\right)\right\}}\,.
\end{eqnarray*}
 in Lemma \ref{lem::Beta-chernoff} yields
%\begin{eqnarray*}
%&&\p\left( 
%R_0(  \hat\phi_k) 
% \geq \frac{a}{m_3+1} + t \right) \\
%&\leq&
%2 \exp\left\{
%- \frac{(m_3+1)^2 t^2}
%{4\left[ (\sqrt{k}-\sqrt{a})t + [a\sqrt{k}+\sqrt{a}\cdot k]/(m_3+1)  \right]^2}
%\right\}\\
%&=:&\delta_3\,.
%\end{eqnarray*}
%where $$g(\delta_3, m_3, k) = (1+\epsilon)\left( \frac{m_3+1-k}{m_3+1}\right)  = 
%\frac{ m_3+1-k}{m_3+1} +  
%\sqrt{\frac{k(m_3+1-k)}{\delta_3(m_3+2)(m_3+1)^2}}\,.$$
%
%Let $a = m_3+1-k,$ $b = k$ 
%Solving the last equality in terms of $t$ leads to
%\begin{eqnarray*}
%&& t = \frac{2\sqrt{\log(2/\delta_3)}\left\{a \sqrt{k}+k \sqrt{a} \right\}}
%{(m_3+1)\left\{m_3+1+ 2\sqrt{\log(2/\delta_3)}\left(\sqrt{a}-\sqrt{k}\right)\right\}}\,.
%\end{eqnarray*}
%Therefore, 
%\begin{eqnarray*}
%%%%%%%%%%%%%%%%%%%%%%%%
%&&\frac{a}{m_3+1} + t =
%\frac{a + 2\sqrt{\log(2/\delta_3)}\cdot\sqrt{a}}
%{m_3+1+ 2\sqrt{\log(2/\delta_3)}\cdot \left(\sqrt{a}-\sqrt{k}\right) }
%= h(\delta_3,m_3,k).
%\end{eqnarray*}
%Thus,
\begin{eqnarray*}
\p\left\{
 R_0(  \hat\phi_k) 
> 
h(\delta_3,m_3,k) \right\}
\,\leq\, \delta_3\,.
\end{eqnarray*}
This completes the proof. 
\end{proof}

\bigskip
Proposition  \ref{prop::general k_chern} implies that  $h(\delta_3,m_3,k) \leq \alpha $ is a sufficient condition for the classifier $\hat{\phi}_k$ (defined in \eqref{eq:threshold-generic}) to satisfy NP Oracle Inequality (I) ($k = 1, \ldots, m_3$). 
Let $\mathcal{K}_\chern = \left\{ k \in \{1,\cdots,m_3\}: h(\delta_3,m_3,k) \leq \alpha \right\}$. 
Similar to Proposition \ref{prop::kmin} we can prove $\mathcal{K}_\chern$ to be non-empty as long as $m_3$ is greater than some threshold. 

Numerical investigation shows that for most combinations of $(\alpha,\delta_3,m_3)$ with non-empty $\mathcal{K}$ and $\mathcal{K}_\chern$,  
$k_{\min}=\min_k\mathcal{K}$ as defined in \eqref{eq::kmin} is better than $k_\chern= \min_k \mathcal{K}_\chern$ in the sense that 
$\hat{\phi}_{k_{\min}}$ has a lower type II error than $\hat{\phi}_{k_{\chern}}$ as a result of $k_{\min} < k_\chern$. 
%%%%%%%%%%%%%%%%%%%%%%%%%%%
Specifically, for each $\delta_3 \in  \{0.01\cdot i\}_{i=1}^{10}$, the number of $\{k_\chern<k_{\min}\}$ out of 100 combinations of $(\alpha, m_3) \in \{0.01\cdot i\}_{i=1}^{10} \times \{100\cdot i\}_{i=1}^{10}$  is reported as follows. Only when $\delta_3$ gets very close to 0 is $k_\chern$ preferred to $k_{\min}$. 
\begin{table}[h]
\begin{center}
\begin{tabular}{l|rrrrrrrrrr}
\hline
$\delta_3$ & 0.01 & 0.02& 0.03& 0.04& 0.05& 0.06& 0.07& 0.08& 0.09& 0.10\\
\#$\{k_\chern<k_{\min}\}$&83  & 70  & 49  & 4  & 0 & 0 &0 &0 &0 &0 \\
\hline
\end{tabular}
\end{center}
\end{table}

\bibliographystyle{ims}
\bibliography{NNB-arXiv}
\end{document}